\documentclass[accepted]{article}

\usepackage{multirow}

\usepackage{mathtools}

\mathtoolsset{showonlyrefs=true}
\usepackage{amssymb}
\usepackage{amsmath}
\usepackage{amsthm}
\usepackage{amsfonts}
\usepackage{ifthen}
\usepackage{graphicx} 

\usepackage{tikz}
\usetikzlibrary{snakes,arrows,shapes}

\usepackage{arydshln}

\usepackage{hyperref}

\usepackage{natbib}
\definecolor{Blue}{rgb}{0,0.16,0.90}
\definecolor{Red}{rgb}{0.50,0.0,0}

\def\E{\mathbb{E}}
\def\Std{\mbox{Std}}
\def\R{\mathbb{R}}
\def\N{\mathbb{N}}
\def\ensconc{[1,\infty) \cup \{\infty\}}
\def\rmax{{R_{\max}}}
\def\vmax{{V_{\max}}}
\def\v*{v_{\pi_*}}
\def\Cpi{C_{\pi_*}}
\def\Ca{C^{(1,0)}}
\def\Cb{C^{(2,1,0)}}
\newcommand\Cam[1]{C^{(1,#1)}}
\newcommand\Cbm[1]{C^{(2,#1,0)}}
\newcommand\Cbbm[1]{C^{(2,#1,m)}}
\def\Cpia{C_{\pi_*}^{(1)}}

\def\greedy{{\mathcal G}}

\usepackage{algorithm}
\usepackage{algorithmic}

\newcommand{\labeldef}[1]{\label{definition:#1}}

\newcommand{\labelsec}[1]{\label{section:#1}}

\newcommand{\refequ}[1]{Equation~\eqref{#1}} 
\newcommand{\refalgo}[1]{\refequ{#1}} 
\newcommand{\refalgoshort}[1]{Eq.~\eqref{#1}} 

\newcommand{\refsec}[1]{Section~\ref{section:#1}}

\newtheorem{theorem}{Theorem}

\newtheorem{corollary}{Corollary}

\newtheorem{definition}{Definition}

\def\lp{\ell_p}

\def\PSDP{PSDP$_\infty$}
\def\NSPI{NSPI($m$)}

\def\titre{Approximate Policy Iteration Schemes: A Comparison}

\def\stretch{}

\usepackage{times}
\usepackage{graphicx} 
\usepackage{subfigure} 

\usepackage{natbib}

\usepackage{algorithm}
\usepackage{algorithmic}

\usepackage{hyperref}

\usepackage{icml2014} 

\icmltitlerunning{\titre}

\begin{document} 

\twocolumn[
\icmltitle{\titre}

\icmlauthor{Bruno Scherrer}{bruno.scherrer@inria.fr}
\icmladdress{Inria, Villers-l\`es-Nancy, F-54600, France\\Universit\'e de Lorraine, LORIA, UMR 7503, Vand{\oe}uvre-l\`es-Nancy, F-54506, France}

\icmlkeywords{Markov Decision Process, Approximate Dynamic Programming, Reinforcement Learning, Performance Bound, Time Complexity, Memory Complexity}

\vskip 0.3in
]

\begin{abstract} 
We consider the infinite-horizon discounted optimal control problem
formalized by Markov Decision Processes. We focus on several
approximate variations of the Policy Iteration algorithm: Approximate Policy Iteration (API)~\citep{ndp}, Conservative Policy Iteration
(CPI)~\citep{Kakade2002}, a natural
adaptation of the Policy Search by
Dynamic Programming algorithm~\citep{Bagnell2003} to the
infinite-horizon case (\PSDP), and the recently proposed
Non-Stationary Policy Iteration (\NSPI)~\citep{Scherrer:2012}. For all
algorithms, we describe performance bounds with respect the per-iteration error $\epsilon$, and
make a comparison by paying a particular attention to the
concentrability constants involved, the number of iterations and the
memory required. Our analysis highlights the following points: 1) The
performance guarantee of CPI can be arbitrarily better than that of
API, but this comes at the cost of a
relative---exponential in $\frac{1}{\epsilon}$---increase of the
number of iterations. 2) PSDP$_\infty$ enjoys the best of both worlds:
its performance guarantee is similar to that of CPI, but within a
number of iterations similar to that of API. 3) Contrary to API that
requires a constant memory, the memory needed by CPI and \PSDP{ }is
proportional to their number of iterations, which may be problematic
when the discount factor $\gamma$ is close to 1 or the
approximation error $\epsilon$ is close to $0$; we show that
the \NSPI{ }algorithm allows to make an overall trade-off between
memory and performance. Simulations with these schemes confirm our
analysis.

\end{abstract}

\section{Introduction}

We consider an infinite-horizon discounted Markov Decision Process (MDP)~\cite{puterman,ndp} $(\mathcal S, \mathcal
A, P, r,\gamma)$, where $\mathcal S$ is a possibly infinite state
space, $\mathcal A$ is a finite action space, $P(ds'|s,a)$, for all
$(s,a)$, is a probability kernel on $\mathcal S$, $r : \mathcal
S \to [-\rmax,\rmax]$ is a reward function bounded by $\rmax$, and $\gamma \in (0,1)$ is a discount factor. A
stationary deterministic policy $\pi:\mathcal S\to\mathcal A$ maps
states to actions. We write $P_\pi(ds'|s)=P(ds'|s,\pi(s))$ for the
stochastic kernel associated to policy $\pi$. The value $v_\pi$ of a
policy $\pi$ is a function mapping states to the expected discounted sum
of rewards received when following $\pi$ from these states: for all
$s\in\mathcal S$, 
$$
v_\pi(s) = \E\left[\sum_{t=0}^\infty \gamma^t
  r(s_t)\middle|s_0=s,~ s_{t+1}\sim
  P_\pi(\cdot|s_t)\right].
$$
The value $v_\pi$ is clearly
bounded by $V_{\mathrm{max}} = R_{\mathrm{max}}/(1-\gamma)$.
It is well-known that $v_\pi$ can be characterized as 
the unique fixed point of the linear Bellman operator
associated to a policy $\pi$: $T_\pi:v \mapsto r + \gamma P_\pi v$.
Similarly, the Bellman optimality operator $T:v \mapsto \max_\pi T_\pi v$ has as unique fixed point the optimal value $v_*=\max_\pi v_\pi$.  A policy $\pi$ is greedy w.r.t. a value
function $v$ if $T_\pi v = T v$, the set of such greedy policies is
written $\greedy{v}$. Finally, a policy $\pi_*$ is optimal, with value
$v_{\pi_*}=v_*$, iff $\pi_*\in\greedy{v_*}$, or equivalently
$T_{\pi_*}v_* = v_*$. 

The goal of this paper is to study and compare several approximate Policy Iteration
schemes.
In the literature, such schemes can be
seen as implementing an approximate greedy operator,
$\greedy_\epsilon$, that takes as input a distribution $\nu$ and a
function $v:S \rightarrow \R$ and returns a policy $\pi$ that is
$(\epsilon,\nu)$-approximately greedy with respect to $v$ in the sense
that:
\begin{equation}
\label{defgreedy}
\nu (T v - T_{\pi} v) = \nu (\max_{\pi'} T_{\pi'} v - T_{\pi} v) \le \epsilon.
\end{equation}
where for all $x$, $\nu x$ denotes $\E_{s\sim\nu}[x(s)]$.
In practice, this approximation of the greedy operator can be achieved
through a $\lp$-regression of the so-called \emph{Q-function}---the
state-action value function---(a direct regression is suggested by \citet{Kakade2002}, a fixed-point LSTD approach is used by~\citet{lagoudakis2003least}) or through a (cost-sensitive)
classification problem~\cite{Lagoudakis:2003b,Lazaric:2010}.
With this operator in hand, we shall describe several Policy Iteration
schemes in \refsec{algos}. Then \refsec{analysis} will provide
a detailed comparative analysis of their performance guarantees, time complexities,
and memory requirements. \refsec{experiments} will go on by providing
experiments that will illustrate their behavior, and confirm
our analysis. Finally, \refsec{conclusion} will conclude and present future work.

\section{Algorithms}
\labelsec{algos}
\paragraph{API}

We begin by describing the standard Approximate Policy Iteration (API)~\citep{ndp}. At each iteration $k$, the algorithm switches to the policy that is approximately greedy with respect to the value of the previous policy for some distribution $\nu$:
\begin{equation}
\label{api}
\pi_{k+1} \leftarrow \greedy_{\epsilon_{k+1}}(\nu, v_{\pi_{k}}).
\end{equation}
If there is no error ($\epsilon_k=0$) and $\nu$ assigns a positive weights to every state, it can easily be seen that this algorithm generates the same sequence of policies as exact Policy Iterations since from \refequ{defgreedy} the policies are exactly greedy. 

\paragraph{CPI/CPI($\alpha$)/API($\alpha$)}

We now turn to the description of Conservative Policy Iteration (CPI) proposed by \cite{Kakade2002}. At iteration $k$, CPI (described in \refalgo{cpi}) uses the distribution $d_{\pi_k,\nu}=(1-\gamma)\nu(I-\gamma P_{\pi_k})^{-1}$---the discounted cumulative occupancy measure induced by $\pi_k$ when starting from $\nu$---for calling the approximate greedy operator, and uses a stepsize $\alpha_k$ to generate a stochastic mixture of all the policies that are returned by the successive calls to the approximate greedy operator, which explains the adjective ``conservative'':
\begin{align}
\label{cpi}
\pi_{k+1}& \leftarrow (1-\alpha_{k+1})\pi_k + \alpha_{k+1} \greedy_{\epsilon_{k+1}}(d_{\pi_k,\nu},v_{\pi_k}) 
\end{align}
The stepsize $\alpha_{k+1}$ can be chosen in such a way that the above step leads to an improvement of the expected value of the policy given that the process is initialized according to the distribution $\nu$ \citep{Kakade2002}. The original article also describes a criterion for deciding whether to stop or to continue. Though the adaptive stepsize and the stopping condition allows to derive a nice analysis, they are in practice conservative: the stepsize $\alpha_k$ should be implemented with a line-search mechanism, or be fixed to some small value $\alpha$. We will refer to this latter variation of CPI as CPI($\alpha$).

It is natural to also consider the algorithm API($\alpha$) (mentioned by \citet{Lagoudakis:2003b}), a variation of API that is conservative like CPI($\alpha$) in the sense that it mixes the new policy with the previous ones with weights $\alpha$ and $1-\alpha$, but that directly uses the distribution $\nu$ in the approximate greedy step:
\begin{align}
\label{apialpha}
\pi_{k+1}& \leftarrow (1-\alpha)\pi_k + \alpha \greedy_{\epsilon_{k+1}}(\nu,v_{\pi_k}) 
\end{align}
Because it uses $\nu$ instead of $d_{\pi_k,\nu}$, API($\alpha$) is simpler to implement than CPI($\alpha$)\footnote{In practice, controlling the greedy step with respect to $d_{\pi_k,\nu}$ requires to generate samples from this very distribution. As explained by \citet{Kakade2002}, one such sample can be done by running one trajectory starting from $\nu$ and following $\pi_k$, stopping at each step with probability $1-\gamma$. In particular, one sample from $d_{\pi_k,\nu}$ requires on average $\frac{1}{1-\gamma}$ samples from the underlying MDP. With this respect, API($\alpha$) is much simpler to implement.}.

\paragraph{\PSDP}

We are now going to describe an algorithm that has a flavour similar
to API---in the sense that at each step it does a full step towards a
new deterministic policy---but also has a conservative flavour like CPI---in the
sense that the policies considered evolve more and more slowly.
This algorithm is a natural variation of the Policy Search by Dynamic
Programming algorithm (PSDP) of \citet{Bagnell2003}, originally proposed
to tackle finite-horizon problems, to the infinite-horizon case; we
thus refer to it as \PSDP. To the best of our knowledge however, this
variation has never been used in an infinite-horizon context.

The algorithm is based on finite-horizon non-stationary policies. Given a sequence of stationary deterministic policies $(\pi_k)$ that the algorithm will generate,  we will write
$\sigma_k=\pi_k \pi_{k-1}\ldots \pi_1$ the $k$-horizon policy that makes the
first action according to $\pi_k$, then the second action according to
$\pi_{k-1}$, etc. Its value is $v_{\sigma_k}=T_{\pi_k}T_{\pi_{k-1}} \ldots T_{\pi_1} r$. 
We will write $\varnothing$ the ``empty''
non-stationary policy. Note that $v_{\varnothing}=r$ and that any
infinite-horizon policy that begins with
$\sigma_k=\pi_{k} \pi_{k-1} \dots \pi_1$, which we will (somewhat abusively) denote
``$\sigma_k \dots$'' has a value $v_{\sigma_k \dots} \ge v_{\sigma_k}-\gamma^k \vmax$.
Starting from $\sigma_0=\varnothing$, the algorithm implicitely builds a sequence of non-stationary policies $(\sigma_k)$  by iteratively concatenating the policies that
are returned by the approximate greedy operator:
\begin{align}
\label{psdp}
\pi_{k+1} & \leftarrow \greedy_{\epsilon_{k+1}}(\nu, v_{\sigma_{k}}) 
\end{align}
While the standard PSDP algorithm of \citet{Bagnell2003} considers a horizon $T$ and makes $T$ iterations, the algorithm we consider here has an \emph{indefinite} number of iterations. 
The algorithm can be stopped at any step $k$. The theory that  we are about to describe suggests that one may return any policy that starts by the non-stationary policy $\sigma_{k}$. Since $\sigma_k$ is an approximately good finite-horizon policy, and as we consider an infinite-horizon problem, a natural output that one may want to use in practice is the infinite-horizon policy that loops over $\sigma_k$, that we shall denote $(\sigma_k)^\infty$.

From a practical point of view, \PSDP{ }and CPI need to store all the
(stationary deterministic) policies generated from the start. The memory required by the
algorithmic scheme is thus proportional to the number of iterations,
which may be prohibitive.
The aim of the next paragraph, that presents the last algorithm of this
article, is to describe a solution to this potential memory issue.

\paragraph{\NSPI}

We originally devised the algorithmic scheme of \refalgo{psdp} (\PSDP) as a simplified variation of
the \emph{Non-Stationary PI algorithm with a growing period} algorithm
(NSPI-growing) \cite{Scherrer:2012}\footnote{We
later realized that it was in fact a very natural variation of
PSDP. To "give Caesar his due and God his", we kept as the main
reference the older work and gave the name \PSDP.}. With
respect to \refalgo{psdp}, the only difference of NSPI-growing resides in the fact that 
the approximate greedy step is done with respect to the value
$v_{(\sigma_k)^\infty}$ of the policy that loops infinitely over
$\sigma_k$ (formally the algorithm does
$\pi_{k+1} \leftarrow \greedy_{\epsilon_{k+1}}(\nu,v_{(\sigma_{k})^\infty})$) 
instead of the value $v_{\sigma_k}$ of only the first $k$ steps here.
Following the intuition that when $k$ is big, these two values will be
close to each other, we ended up considering \PSDP{ }because it is
simpler.  
NSPI-growing suffers from the
same memory drawback as CPI and \PSDP. Interestingly, the work of \citet{Scherrer:2012} contains
another algorithm, \emph{Non-Stationary PI with a fixed period}
(\NSPI), that has a parameter that directly controls
the number of policies stored in memory.

Similarly to \PSDP, \NSPI{ }is based on non-stationary policies.  It
takes as an input a parameter $m$. It requires a set of $m$ initial deterministic stationary policies
$\pi_{m-1},\pi_{m-2},\dots,\pi_0$ and iteratively generates new
policies $\pi_1,\pi_2,\dots$. 
For any $k \ge 0$, we shall denote $\sigma_k^m$ the $m$-horizon non-stationary policy
that runs \emph{in reverse order} the last $m$ policies, which one may write formally:
$
\sigma_k^m = 
{\pi_k ~ \pi_{k-1} ~ \dots ~ \pi_{k-m+1}}
.
$
Also, we shall denote
$(\sigma_k^m)^\infty$ the $m$-periodic infinite-horizon non-stationary policy that loops over $\sigma_k^m$. Starting from $\sigma_0^m=\pi_0 \pi_1 \dots \pi_{m-1}$, the algorithm iterates as follows:
\begin{align}
\label{nspi}
\pi_{k+1} & \leftarrow \greedy_{\epsilon_{k+1}}(\nu, v_{(\sigma_k^m)^\infty})
\end{align}
Each iteration requires to compute an approximate greedy
policy $\pi_{k+1}$ with respect to the value $v_{(\sigma_k^m)^\infty}$ of $(\sigma_k^m)^\infty$,
that is the fixed point of the compound operator\footnote{Implementing this algorithm in practice can trivially be done through cost-sensitive classification in a way similar to \citet{Lazaric:2010}. It could also be done with a straight-forward extension of LSTD($\lambda$) to non-stationary policies.}:
$$
\forall v,~T_{k,m} v= T_{\pi_k} T_{\pi_{k-1}} \dots T_{\pi_{k-m+1}}v.
$$
When one goes from iterations $k$ to $k+1$, the process consists in adding $\pi_{k+1}$ at the front of the $(m-1)$-horizon policy 
$\pi_{k}\pi_{k-1}\dots \pi_{k-m+2}$, thus forming a new $m$-horizon policy $\sigma_{k+1}^m$. Doing so, we
forget about the oldest policy $\pi_{k-m+1}$ of $\sigma_k^m$ and keep
a constant memory of size $m$. At any step $k$, the algorithm can be stopped, and
the output is the policy $\pi_{k,m}=(\sigma_k^m)^\infty$ that loops on $\sigma_k^m$.
It is easy to see that \NSPI{ }reduces to API when
$m=1$. Furthermore, if we assume that the reward function is positive, add ``stop actions'' in every state of the model that
lead to a terminal absorbing state with a null reward, and
initialize with an infinite sequence of policies that only take
this ``stop action'', then \NSPI{ }with $m=\infty$ reduces to \PSDP.

\begin{table*}
\begin{center}
\begin{tabular}{|c||ccc|c|c||c|}
\hline 
 Algorithm  &  \multicolumn{3}{c|}{Performance Bound} & {$\#$ Iter.}& Memory & Reference  \\
\hline
\hline
{API (\refalgoshort{api})} & $\Cb$ & $\frac{1}{(1-\gamma)^2}$ & $\epsilon$  & \multirow{2}{*}{$ \frac{1}{1-\gamma}{\log{\frac 1 \epsilon}}$} & \multirow{2}{*}{$1$} & {\small \citep{Lazaric:2010}} \\
{\scriptsize($=$~NSPI(1))}&  {$\Ca$} & {$\frac{1}{(1-\gamma)^2}$} & {$\epsilon \log{\frac 1 \epsilon}$}    & & &  \\
\hdashline
{API($\alpha$) (\refalgoshort{apialpha}} 
& {$\Ca$} & {$ \frac{1}{(1-\gamma)^2} $} & {$ \epsilon$} &   \multicolumn{2}{c||}{$ \frac 1 {\alpha(1-\gamma)}  \log{\frac 1 \epsilon}$} &  \\
\hline
{CPI($\alpha$)} 
& {$\Ca$} & {$ \frac{1}{(1-\gamma)^3} $} & {$ \epsilon$} &   \multicolumn{2}{c||}{$ \frac 1 {\alpha(1-\gamma)}  \log{\frac 1 \epsilon}$} &  \\
\hdashline
\multirow{2}{*}{CPI (\refalgoshort{cpi})}  
& {$\Ca$} & {$\frac{1}{(1-\gamma)^3}$} & {$\epsilon \log{\frac 1 \epsilon}$} &  \multicolumn{2}{c||}{$\frac{1}{1-\gamma} \frac 1 \epsilon  \log{\frac 1 \epsilon} $} & \\
 & $\Cpi$ & $\frac{1}{(1-\gamma)^2}$ & $\epsilon $ & \multicolumn{2}{c||}{$\frac{\gamma}{\epsilon^2}$} & {\small \citep{Kakade2002}} \\ 
\hline
{\PSDP{ }(\refalgoshort{psdp})} & {$\Cpi$} & {$\frac{1}{(1-\gamma)^2}$} & {$\epsilon  \log{\frac 1 \epsilon}$}  & \multicolumn{2}{c||}{$ \frac{1}{1-\gamma}\log{\frac 1 \epsilon}$} & \\
{\scriptsize($\simeq$~NSPI($\infty$))}&  {$\Cpia$} & {$\frac{1}{1-\gamma}$} & {$\epsilon$}  & \multicolumn{2}{c||}{$ \frac{1}{1-\gamma}\log{\frac 1 \epsilon}$} & 
 \\
\hline
\multirow{4}{*}{\NSPI{ }(\refalgoshort{nspi})} 
 &  {$\Cbm{m}$} & {$\frac{1}{(1-\gamma)(1-\gamma^m)}$} & {$\epsilon$}  & {{$\frac{1}{1-\gamma} \log{\frac 1 \epsilon}$}} &  \multirow{4}{*}{$m$} & \multirow{4}{*}{}\\
&  {$\frac{\Ca}{m}$} & {$\frac{1}{(1-\gamma)^2(1-\gamma^m)}$} & {$\epsilon \log{\frac 1 \epsilon}$} &  {{$\frac{1}{1-\gamma} \log{\frac 1 \epsilon}$}} & & \\
& {{$\Cpia+ \gamma^m\frac{\Cbbm{m}}{1-\gamma^m} $}} & {$\frac{1}{1-\gamma}$} & {$\epsilon$}&   {{$\frac{1}{1-\gamma} \log{\frac 1 \epsilon}$}}& & \\
& {$\Cpi+\gamma^m \frac{\Cbm{m}}{m(1-\gamma^m)}$} & {$\frac{1}{(1-\gamma)^2}$} & {$\epsilon  \log{\frac 1 \epsilon}$} & {{$\frac{1}{1-\gamma} \log{\frac 1 \epsilon}$}}& & \\
\hline
\end{tabular}

\end{center}
\caption{\label{fig:comparison}{{\bf Upper bounds on the performance guarantees for the algorithms.} Except when references are given, the bounds are to our knowledge new. A comparison of API and CPI based on the two known bounds was done by \citet{Ghavam:2012}. The first bound of \NSPI{ }can be seen as an adaptation of that provided by \citet{Scherrer:2012}} for the more restrictive $\ell_\infty$-norm setting.}
\end{table*}

\section{Analysis}
\labelsec{analysis}

For all considered algorithms, we are going to describe bounds on the expected loss $E_{s \sim \mu}[\v*(s) - v_{\pi}(s)]=\mu(\v*-v_{\pi})$  of using the (possibly stochastic or non-stationary) policy $\pi$ ouput by the algorithms instead of the optimal policy $\pi_*$ from some initial distribution $\mu$ of interest as a function of an upper bound $\epsilon$ on all errors $(\epsilon_k)$. In order to derive these theoretical guarantees, we will first need to introduce a few concentrability coefficients that relate the distribution $\mu$ with which one wants to have a guarantee, and the distribution $\nu$ used by the algorithms\footnote{The expected loss corresponds to some weighted $\ell_1$-norm of the loss $\v*-v_{\pi}$. Relaxing the goal to controlling the weighted $\ell_p$-norm for some $p \ge 2$ allows to introduce some finer coefficients~\cite{FaMuSz10,ampi}. Due to lack of space, we do not consider this here.}.
\begin{definition}
\labeldef{conc}
Let $c(1), c(2),\ldots$ be the smallest coefficients in $\ensconc$ such that for all $i$ and all sets of deterministic stationary policies $\pi_1,\pi_2,\ldots, \pi_i$, $\mu P_{\pi_1} P_{\pi_2} \ldots P_{\pi_i} \le c(i) \nu$. For all $m,k$, we define the following coefficients in $\ensconc$:

~\vspace{-.7cm}
\begin{align}
\Cam{k} & = (1-\gamma)\sum_{i=0}^{\infty} \gamma^i c(i+k), \\
C^{(2,m,k)} & = (1-\gamma)(1-\gamma^m) \sum_{i=0}^{\infty} \sum_{j=0}^\infty \gamma^{i+jm} c(i+jm+k).
\end{align}
~\vspace{-1cm}

Similarly, let $c_{\pi_*}(1), c_{\pi_*}(2),\ldots$ be the smallest coefficients in $\ensconc$ such that for all $i$, $\mu (P_{\pi_*})^i \le c_{\pi_*}(i) \nu$. We define:

~\vspace{-.9cm}
\begin{align}
\Cpia & = (1-\gamma)\sum_{i=0}^{\infty} \gamma^i c_{\pi_*}(i).
\end{align}
Finally let $\Cpi$ be the smallest coefficient in $\ensconc$ such that $d_{\pi_*,\mu} =(1-\gamma)\mu (I-\gamma P_{\pi_*})^{-1} \le \Cpi\nu$.
\end{definition}
~\vspace{-.7cm}

With these notations in hand, our first contribution is to
provide a thorough comparison of all the algorithms. This is done in
Table~\ref{fig:comparison}. For each algorithm, we describe some
performance bounds and the required number of iterations and memory.
To make things clear, we only display the dependence with respect to
the concentrability constants, the discount factor $\gamma$, the quality
$\epsilon$ of the approximate greedy operator, and---if
applicable---the main parameters $\alpha$/$m$ of the algorithms.  For
API($\alpha$), CPI($\alpha$), CPI and \PSDP, the required memory
matches the number of iterations. All but two bounds are to
our knowledge original. The derivation of the new results are given in
Appendix~\ref{bounds}.

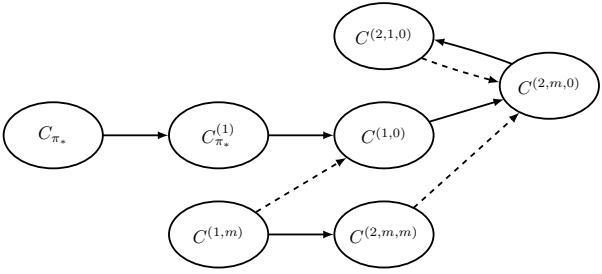
\begin{figure}
\begin{center}
\resizebox{8cm}{!}{\begin{tikzpicture}[>=latex,line join=bevel,]
  \pgfsetlinewidth{1bp}
\pgfsetcolor{black}
  \draw [->,dashed] (137.42bp,30.253bp) .. controls (149.39bp,37.434bp) and (164.7bp,46.621bp)  .. (186.66bp,59.795bp);
  \draw [->,dashed] (223.09bp,32.481bp) .. controls (236.95bp,44.959bp) and (257.17bp,63.15bp)  .. (280.75bp,84.371bp);
  \draw [->] (276.6bp,110.88bp) .. controls (266.95bp,114.64bp) and (255.12bp,118.49bp)  .. (234bp,124.14bp);
  \draw [->] (232.07bp,79.521bp) .. controls (241.51bp,82.353bp) and (252.42bp,85.625bp)  .. (272.38bp,91.613bp);
  \draw [->] (54.003bp,72bp) .. controls (62.028bp,72bp) and (70.967bp,72bp)  .. (89.705bp,72bp);
  \draw [->] (144bp,72bp) .. controls (152.03bp,72bp) and (160.97bp,72bp)  .. (179.71bp,72bp);
  \draw [->] (144bp,18bp) .. controls (152.03bp,18bp) and (160.97bp,18bp)  .. (179.71bp,18bp);
  \draw [->,dashed] (227.42bp,114.11bp) .. controls (237.03bp,110.37bp) and (248.8bp,106.54bp)  .. (269.83bp,100.9bp);
\begin{scope}
  \definecolor{strokecol}{rgb}{0.0,0.0,0.0};
  \pgfsetstrokecolor{strokecol}
  \draw (207bp,18bp) ellipse (27bp and 18bp);
  \draw (207bp,18bp) node {$C^{(2,m,m)}$};
\end{scope}
\begin{scope}
  \definecolor{strokecol}{rgb}{0.0,0.0,0.0};
  \pgfsetstrokecolor{strokecol}
  \draw (117bp,18bp) ellipse (27bp and 18bp);
  \draw (117bp,18bp) node {$C^{(1,m)}$};
\end{scope}
\begin{scope}
  \definecolor{strokecol}{rgb}{0.0,0.0,0.0};
  \pgfsetstrokecolor{strokecol}
  \draw (117bp,72bp) ellipse (27bp and 18bp);
  \draw (117bp,72bp) node {$C^{(1)}_{\pi_*}$};
\end{scope}
\begin{scope}
  \definecolor{strokecol}{rgb}{0.0,0.0,0.0};
  \pgfsetstrokecolor{strokecol}
  \draw (207bp,72bp) ellipse (27bp and 18bp);
  \draw (207bp,72bp) node {$C^{(1,0)}$};
\end{scope}
\begin{scope}
  \definecolor{strokecol}{rgb}{0.0,0.0,0.0};
  \pgfsetstrokecolor{strokecol}
  \draw (27bp,72bp) ellipse (27bp and 18bp);
  \draw (27bp,72bp) node {$C_{\pi_*}$};
\end{scope}
\begin{scope}
  \definecolor{strokecol}{rgb}{0.0,0.0,0.0};
  \pgfsetstrokecolor{strokecol}
  \draw (207bp,126bp) ellipse (27bp and 18bp);
  \draw (207bp,126bp) node {$C^{(2,1,0)}$};
\end{scope}
\begin{scope}
  \definecolor{strokecol}{rgb}{0.0,0.0,0.0};
  \pgfsetstrokecolor{strokecol}
  \draw (297bp,99bp) ellipse (27bp and 18bp);
  \draw (297bp,99bp) node {$C^{(2,m,0)}$};
\end{scope}
\end{tikzpicture}}
\end{center}
\caption{\label{fig:constants}{\bf Hierarchy of the concentrability constants.} A constant $A$ is better than a constant $B$---see the text for details---if $A$ is a parent of $B$ on the above graph. The best constant is $\Cpi$.}
\end{figure}

Our second contribution, that is complementary with the
comparative list of bounds, is that we can show that there exists a
hierarchy among the constants that appear in all the bounds of
Table~\ref{fig:comparison}. In the directed graph of
Figure~\ref{fig:constants}, a constant $B$ is a descendent of
$A$ \emph{if and only if} the implication $\{B<\infty \Rightarrow
A<\infty\}$ holds\footnote{Dotted arrows are used to underline the fact
that the comparison of coefficients is restricted to the case where the parameter $m$ is
finite.}. The ``if and only if'' is important here: it means that if
$A$ is a parent of $B$, and $B$ is not a parent of $A$, then there
exists an MDP for which $A$ is finite while $B$ is infinite; in other words, an algorithm that has a guarantee with respect to $A$ has a guarantee that can be arbitrarily better than that with constant $B$. Thus, the
overall best concentrability constant is $\Cpi$, while the worst are
$\Cb$ and $\Cbm{m}$. To make the picture complete, we should add that
for any MDP and any distribution $\mu$, it is possible to find an
input distribution $\nu$ for the algorithm (recall that the
concentrability coefficients depend on $\nu$ and $\mu$) such that
$\Cpi$ is finite, though it is not the case for $\Cpia$ (and as a
consequence all the other coefficients). The
derivation of this order relations is done in Appendix~\ref{coefcompar}.

The standard API algorithm has guarantees expressed in terms of $\Cb$
and $\Ca$ only. Since CPI's analysis can be done with respect to
$\Cpi$, it has a performance guarantee that can be arbitrarily better
than that of API, though the opposite is not true. This, however,
comes at the cost of an exponential increase of time
complexity since CPI may require a number of iterations that scales in
$O\left(\frac{1}{\epsilon^2}\right)$, while the guarantee of API only requires
$O\left(\log \frac 1 \epsilon \right)$ iterations. When the analysis
of CPI is relaxed so that the performance guarantee is expressed in
terms of the (worse) coefficient $\Ca$ (obtained also for API), we can slightly improve
the rate---to $\tilde O\left(\frac{1}{\epsilon}\right)$---, though it is still
exponentially slower than that of API. This second result for CPI was
proved with a technique that was also used for CPI($\alpha$) and
API($\alpha$). We  conjecture that it can be improved for CPI($\alpha$), that
should be as good as CPI when $\alpha$ is sufficiently small.

\PSDP{ }enjoys two guarantees that have a fast rate like those of API.
One bound has a better dependency with respect to
$\frac{1}{1-\gamma}$, but is expressed in terms of the worse
coefficient $\Cpia$.  The second guarantee is almost as good as that
of CPI since it only contains an extra $\log\frac 1 \epsilon$ term,
but it has the nice property that it holds quickly with respect to~$\epsilon$: in time $O (\log \frac 1 \epsilon)$ instead of $O(\frac 1
{\epsilon^2})$, that is exponentially faster. \PSDP{ }is thus
theoretically better than both CPI (as good but faster) and API (better
and as fast).

\def\taille{0.95}

\newcommand{\showcurve}[3]{
\begin{figure*}
\begin{center}
\includegraphics[width=\taille\textwidth]{#1.pdf}
\end{center}
\caption{#2 \label{#3}}
\end{figure*}
}

\showcurve{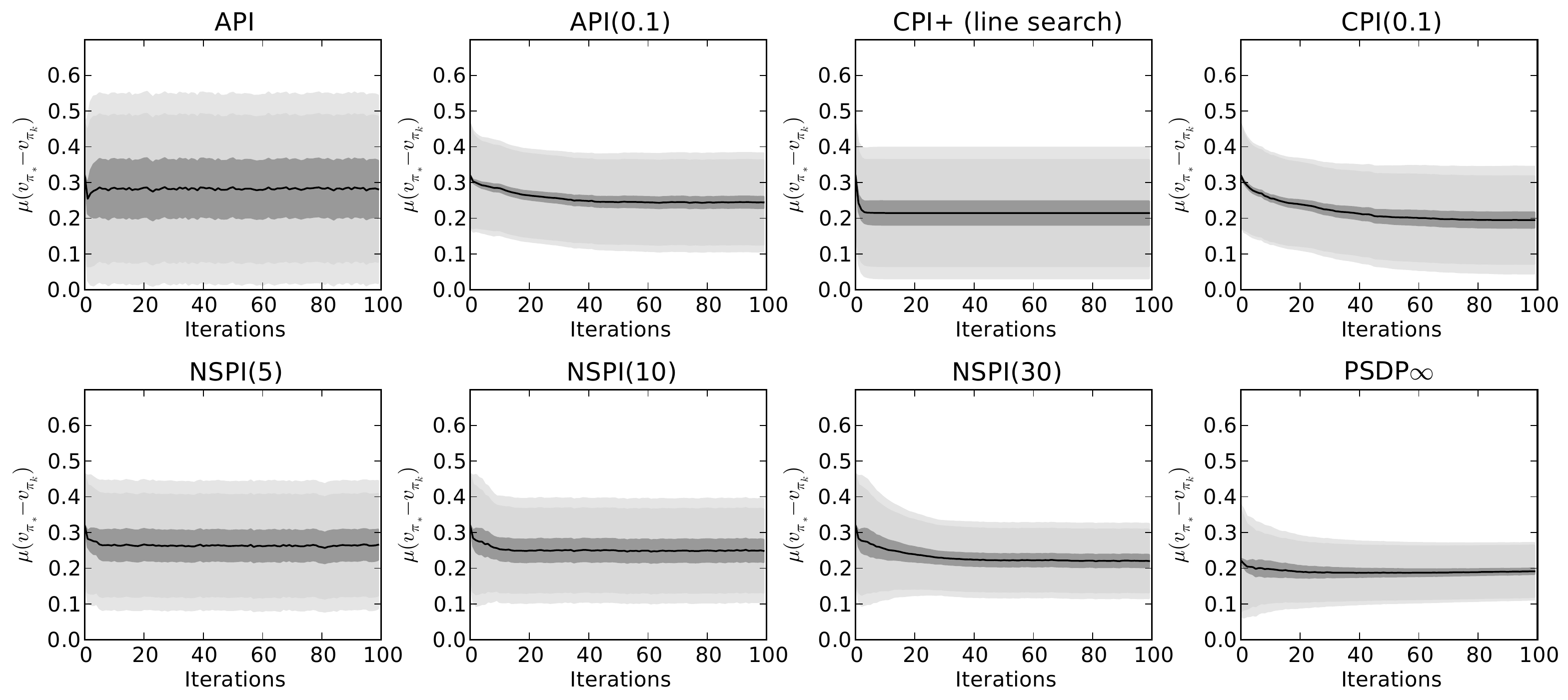}{{\bf Statistics for all instances.} 
The MDPs $(M_i)_{1 \le i \le 30}$ are i.i.d. with the same distribution as
  $M_1$. Conditioned on some MDP $M_i$ and some algorithm, the error
  measures at all iteration $k$ are i.i.d. with the same distribution
  as $L_{1,k}$. The central line of the learning curves gives the
  empirical estimate of the overall average error
  $(\E[L_{1,k}])_k$. The three grey regions (from dark to light grey)
  are estimates of respectively the variability (across MDPs) of the
  average error $(\Std[\E[L_{1,k}|M_1]])_k$, the average (across MDPs)
  of the standard deviation of the error $(\E[\Std[L_{1,k}|M_1]])_k$,
  and the variability (across MDPs) of the standard deviation of the
  error $(\Std[\Std[L_{1,k}|M_1]])_k$.  For ease of comparison, all
  curves are displayed with the same $x$ and $y$ range. }{expall}

Now, from a practical point of view, \PSDP{ }and CPI need to store all
the policies generated from the start. The memory required by these
algorithms is thus proportional to the number of iterations.  Even
if \PSDP{ }may require much fewer iterations than CPI, the
corresponding memory requirement may still be prohibitive in
situations where $\epsilon$ is small or $\gamma$ is close to $1$.  We
explained that \NSPI{ }can be seen as making a bridge between API
and \PSDP. Since (i) both have a nice time complexity, (ii) API has
the best memory requirement, and (iii) \NSPI{ } has the best
performance guarantee, \NSPI{ }is a
good candidate for making a standard performance/memory trade-off.  If the
first two bounds of \NSPI{ }in Table~\ref{fig:comparison} extends
those of API, the other two are made of two terms: the left terms are
identical to those obtained for \PSDP, while the two possible right
terms are new, but are controlled by $\gamma^m$,
which can thus be made arbitrarily small by increasing the memory
parameter $m$. Our analysis thus confirms our intuition that \NSPI{ }
allows to make a performance/memory trade-off in between API (small
memory) and \PSDP{ }(best performance). In other words, as soon as
memory becomes a constraint, \NSPI{ }is the natural alternative to \PSDP.

\stretch
\section{Experiments}
\stretch

\labelsec{experiments}

In this section, we present some experiments in order to illustrate
the empirical behavior of the different algorithms discussed in the paper.  We considered the standard API
as a baseline.  CPI, as it is described by \citet{Kakade2002}, is very slow
(in one sample experiment on a 100 state problem, it made very slow
progress and took several millions of iterations before it stopped)
and we did not evaluate it further. Instead, we considered two
variations: CPI+ that is identical to CPI except that it chooses the
step $\alpha_k$ at each iteration by doing a line-search towards the
policy output by the greedy operator\footnote{We implemented a crude
  line-search mechanism, that looks on the set $2^i \alpha$ where
  $\alpha$ is the minimal step estimated by CPI to ensure
  improvement.}, and CPI($\alpha$) with $\alpha=0.1$, that makes
``relatively but not too small'' steps at each iteration. To assess
the utility for CPI to use the distribution $d_{\nu,\pi}$ for the
approximate greedy step, we also considered API($\alpha$) with
$\alpha=0.1$, the variation of API described in \refalgo{apialpha}
that makes small steps, and that only differs from CPI($\alpha$) by
the fact that the approximate greedy step uses the distribution $\nu$
instead of $d_{\pi_k,\nu}$.  In addition to these algorithms, we
considered \PSDP{ }and \NSPI{ }for the values $m \in \{5,10,30\}$.

In order to assess their quality, we consider finite problems where
the exact value function can be computed. More precisely, we consider
Garnet problems first introduced by \citet{Archibald:95}, which are a
class of randomly constructed finite MDPs. They do not correspond to
any specific application, but remain
representative of the kind of MDP that might be encountered in
practice. In brief,
we consider Garnet problems with $|{\cal S}| \in \{50,100,200\}$, $|{\cal A}|
\in \{2,5,10\}$ and branching factors in $\{1,2,10\}$. The greedy step
used by all algorithms is approximated by an exact greedy operator
applied to a noisy orthogonal projection on a linear space of
dimension $\frac{|{\cal S}|}{10}$ with respect to the quadratic norm weighted
by $\nu$ or $d_{\nu,\pi}$ (for CPI+ and CPI($\alpha$)) where $\nu$ is uniform. 

For each of these $3^3=27$ parameter instances, we generated 30
i.i.d. Garnet MDPs $(M_i)_{1 \le i \le 30}$.  For each such MDP $M_i$,
we ran API, API(0.1), CPI+, CPI(0.1), \NSPI{ }for $m \in \{5,10,30\}$
and \PSDP{ }$30$ times. For each run $j$ and algorithm, we compute for
all iterations $k \in (1,100) $ the performance,
i.e. the loss $L_{j,k}=\mu(\v*-v_{\pi_k})$ with respect to the optimal
policy.  Figure~\ref{expall} displays statistics about
these random variables.  For each algorithm, we display a learning
curve with confidence regions that account for the variability across runs and problems.  
The supplementary material contains statistics that are respectively conditioned on the values of $n_S$, $n_A$ and $b$, which gives some insight
on the influence of these parameters.

From these experiments and statistics, we can make a series of
observations.
The standard API scheme is much more variable than the other
algorithms and tends to provide the worst performance on average.
CPI+ and CPI($\alpha$) display about the same asymptotic performance
on average. If CPI($\alpha$) has slightly less variability, it is much
slower than CPI+, that always converges in very few iterations (most
of the time less than 10, and always less than 20).
API($\alpha$)---the naive conservative variation of API that is also
simpler than CPI($\alpha$)---is empirically close to CPI($\alpha$),
while being on average slightly worse.
CPI+, CPI($\alpha$) and \PSDP{ }have a similar average performance,
but the variability of \PSDP{ }is significantly smaller. \PSDP{ }is the
algorithm that overall gives the best results.
\NSPI{ }does indeed provide a bridge between API and \PSDP. By increasing $m$, the behavior
gets closer to that of \PSDP. With $m=30$, \NSPI{ }is overall better
than API($\alpha$), CPI+, and CPI($\alpha$), and close to \PSDP.
The above relative observations are stable with respect to the
number of states $n_S$ and actions $n_A$. Interestingly, the differences between the
algorithms tend to vanish when the dynamics of the problem gets more
and more stochastic (when the branching factor increases). This complies with our analysis based on concentrability coefficients: there are all finite
when the dynamics mixes a lot, and their relative difference are the
biggest in deterministic instances.

\stretch
\section{Discussion, Summary and Future Work}
\stretch
\labelsec{conclusion}

We have considered several variations of the Policy
Iteration schemes for infinite-horizon problems: API, CPI, \NSPI, API($\alpha$) and \PSDP\footnote{We recall that to our knowledge, the use of \PSDP{ }(PSDP in an \emph{infinite-horizon} context) is not documented in the literature.}.
We have in particular explained the fact---to our knowledge so far
unknown---that the recently introduced \NSPI{ }algorithm
generalizes API (that is obtained when $m$=1) and \PSDP{ }(that is
very similar when $m=\infty$). Figure~\ref{fig:comparison}
synthesized the theoretical guarantees about these
algorithms. Most of the bounds are to our knowledge
new.

One of the first important message of our work is that what is usually
hidden in the constants of the performance bounds does matter. The
constants involved in the bounds for API, CPI, \PSDP{ }and for the
main (left) terms of \NSPI{ }can be sorted from the worst to the best
as follows: $\Cb, \Ca, \Cpia, \Cpi$.  A detailed hierarchy of all
constants was depicted in Figure~\ref{fig:constants}.  This is to our
knowledge the first time that such an in-depth comparison of the
bounds is done, and our hierarchy of constants has
interesting implications that go beyond the Policy Iteration schemes
we have been focusing on in this paper.  As a matter of fact, several
other dynamic programming algorithms, namely AVI~\citep{Munos_SIAM07},
$\lambda$PI~\cite{lpi}, AMPI~\citep{ampi}, come with guarantees
involving the worst constant $\Cb$, which suggests that they should
not be competitive with the best algorithms we have described here.

At the purely technical level, several of our bounds come in pair; this
is due to the fact that we have introduced a new proof technique.
This led to a new bound for API, that improves the state of the art in the sense that it involves the constant
$\Ca$  instead of $\Cb$.  It also enabled us to derive new
bounds for CPI (and its natural algorithmic variant CPI($\alpha$))
that is worse in terms of guarantee but has a better time complexity
($\tilde O(\frac{1}{\epsilon})$ instead of $O(\frac{1}{\epsilon^2})$).
We believe this new technique may be helpful in the future for the
analysis of other MDP algorithms.

Let us sum up the main insights of our analysis.  1) The guarantee for
CPI can be arbitrarily stronger than that of API/API($\alpha$),
because it is expressed with respect to the best concentrability
constant $\Cpi$, but this comes at the cost of a
relative---exponential in $\frac{1}{\epsilon}$---increase of the
number of iterations. 2) PSDP$_\infty$ enjoys the best of both worlds:
its performance guarantee is similar to that of CPI, but within a
number of iterations similar to that of API. 3) Contrary to API that
requires a constant memory, the memory needed by CPI and \PSDP{ }is
proportional to their number of iterations, which may be problematic
in particular when the discount factor $\gamma$ is close to $1$ or the
approximation error $\epsilon$ is close to $0$; we showed that
the \NSPI{ }algorithm allows to make an overall trade-off between
memory and performance.

The main assumption of this work is that all algorithms have at disposal an
$\epsilon$-approximate greedy operator. It may be unreasonable to compare all
algorithms on this basis, since the underlying optimization problems
may have different complexities: for instance, methods
like CPI look in a space of stochastic policies while API moves in a
space of deterministic policies. Digging and understanding in more
depth what is potentially hidden in the term $\epsilon$---as we have
done here for the concentrability constants---constitutes a very
natural research direction.

Last but not least, we have run 
numerical experiments that support our worst-case analysis. 
On simulations on about $800$ Garnet MDPs with various
characteristics, CPI($\alpha$), CPI+ (CPI with a crude line-search
mechanism), \PSDP{ }and \NSPI{ }were shown to always perform
significantly better than the standard API. CPI+, CPI($\alpha$)
and \PSDP{ }performed similarly on average, but \PSDP{ }showed much
less variability and is thus the best algorithm in terms of overall
performance. Finally, \NSPI{ }allows to make a bridge between
API and \PSDP, reaching an overall performance close to that of \PSDP{ }
with a controlled memory.
Implementing other instances of these algorithmic schemes, running and
analyzing experiments on bigger domains constitutes interesting future
work.

\appendix
\section{Proofs for Table~\ref{fig:comparison}}
\label{bounds}

\paragraph{\underline{\PSDP:}}
For all $k$, we have
\begin{align}
v_{\pi_*} - v_{\sigma_{k}} &
=~ T_{\pi_*} v_{\pi_*}- T_{\pi_*}v_{\sigma_{k-1}} + T_{\pi_*}v_{\sigma_{k-1}} - T_{\pi_k}v_{\sigma_{k-1}} \\
&
\le~ \gamma P_{\pi_*} (v_\pi-v_{\sigma_{k-1}}) + e_k \label{eq0}
\end{align}
where we defined $e_k=\max_{\pi'}T_{\pi'}v_{\sigma_{k-1}} - T_{\pi_k}v_{\sigma_{k-1}}$.
As $P_{\pi_*}$ is non negative, we deduce by induction:
\begin{align}
v_{\pi_*}- v_{\sigma_k} \le~ \sum_{i=0}^{k-1} (\gamma P_{\pi_*})^i e_{k-i} + \gamma^k \vmax. 
\end{align}
By multiplying both sides by $\mu$, using the definition of the coefficients $c_{\pi_*}(i)$ and the fact that $\nu e_j \le \epsilon_j \le \epsilon$, we get:
\begin{align}
\mu(v_{\pi_*}- v_{\sigma_k}) &
\le~ \sum_{i=0}^{k-1} \mu (\gamma P_{\pi_*})^i e_{k-i} + \gamma^k \vmax \label{eq1} \\
& 
\le~ \sum_{i=0}^{k-1} \gamma^i c_{\pi_*}(i) \epsilon_{k-i} + \gamma^k \vmax \\
&
\le~ \left(\sum_{i=0}^{k-1} \gamma^i c_{\pi_*}(i)\right) \epsilon + \gamma^k \vmax.
\end{align}
The bound with respect to $\Cpia$ is obtained by using the fact that $v_{\sigma_k \dots} \ge v_{\sigma_k} - \gamma^k \vmax$ and taking $k\ge \left \lceil \frac{\log{\frac{2\vmax}{\epsilon}}}{1-\gamma} \right \rceil$.
Starting back in \refequ{eq1} and using the definition of $\Cpi$ (in particular the fact that for all $i$, $\mu(\gamma P_{\pi_*})^i \le \frac{1}{1-\gamma}d_{\pi^*,\mu} \le \frac{\Cpi}{1-\gamma}\nu$) and the fact that $\nu e_j \le \epsilon_j$, we get:
\begin{align}
\mu(v_{\pi_*}- v_{\sigma_k}) &
\le \sum_{i=0}^{k-1} \mu (\gamma P_{\pi_*})^i e_{k-i} + \gamma^k \vmax \\
&
\le \frac{\Cpi}{1-\gamma}\sum_{i=1}^{k}\epsilon_i + \gamma^k \vmax
\end{align}
and the other bound is obtained by using the fact that $v_{\sigma_k \dots} \ge v_{\sigma_k} - \gamma^k \vmax$, $\sum_{i=1}^{k}\epsilon_i \le k \epsilon$, and considering the number of iterations 
$k=\left \lceil \frac{\log{\frac{2\vmax}{\epsilon}}}{1-\gamma} \right \rceil$.

\paragraph{\underline{API/\NSPI:}} 
API is identical to NSPI(1), and its bounds are particular cases of the first two bounds for \NSPI, so we only consider \NSPI.
By following the proof technique of \citet{Scherrer:2012}, writing $\Gamma_{k,m}=(\gamma P_{\pi_k}) (\gamma P_{\pi_{k-1}}) \cdots (\gamma P_{\pi_{k-m+1}})$ and $e_{k+1}=\max_{\pi'}T_{\pi'}v_{\pi_{k,m}} - T_{\pi_{k+1}}v_{\pi_{k,m}}$, one can show that:
\begin{align}
v_{\pi_*}-v_{\pi_{k,m}} &\le \sum_{i=0}^{k-1} (\gamma P_{\pi_*})^i (I-\Gamma_{k-i,m})^{-1} e_{k-i}   + \gamma^k \vmax.
\end{align}
Multiplying both sides by $\mu$ (and observing that $e_k \ge 0$) and the fact that $\nu e_j \le \epsilon_j \le \epsilon$, we obtain:
\begin{align}
&\mu(v_{\pi_*}-v_{\pi_{k}})\\
& \le \sum_{i=0}^{k-1} \mu (\gamma P_{\pi_*})^i(I-\Gamma_{k-i,m})^{-1} e_{k-i}  + \gamma^k \vmax \label{nspi:eq4}\\
& \le  \sum_{i=0}^{k-1} \left( \sum_{j=0}^\infty \gamma^{i+jm} c(i+jm) \epsilon_{k-i} \right) + \gamma^k \vmax  \label{nspi:eq3} \\
& \le  \sum_{i=0}^{k-1} \sum_{j=0}^\infty \gamma^{i+jm} c(i+jm) \epsilon + \gamma^k \vmax, \label{nspi:eq5}
\end{align}
which leads to the first bound by taking $k\ge \left \lceil \frac{\log{\frac{2\vmax}{\epsilon}}}{1-\gamma} \right \rceil$.
Starting back on \refequ{nspi:eq3}, assuming for simplicity  that  $\epsilon_{-k}=0$ for all $k \ge 0$, we get:
{\small
\begin{align}
&\mu(v_{\pi_*}-v_{\pi_{k}})-\gamma^k \vmax\\
 & \le \sum_{l=0}^{\left\lceil \frac{k-1}{m}\right\rceil}\sum_{h=0}^{m-1}  \sum_{j=0}^\infty  \gamma^{h+(l+j)m} c(h+(l+j)m) \epsilon_{k-h-lm}  \\ 
& \le \sum_{l=0}^{\left\lceil \frac{k-1}{m}\right\rceil}\sum_{h=0}^{m-1}  \sum_{j=l}^\infty  \gamma^{h+jm} c(h+jm) \max_{k-(l+1)m+1 \le p \le k-lm} \epsilon_p \\ 
& \le \sum_{l=0}^{\left\lceil \frac{k-1}{m}\right\rceil}\sum_{h=0}^{m-1}  \sum_{j=0}^\infty  \gamma^{h+jm} c(h+jm) \max_{k-(l+1)m+1 \le p \le k-lm} \epsilon_p
\end{align}
\begin{align}
& = \left( \sum_{h=0}^{m-1}  \sum_{j=0}^\infty  \gamma^{h+jm} c(h+jm) \right)  \sum_{l=0}^{\left\lceil \frac{k-1}{m}\right\rceil}\max_{l-(l+1)m+1 \le p \le k-lm} \epsilon_p  \\
& \le \left( \sum_{i=0}^{\infty}   \gamma^{i} c(i) \right) {\left\lceil\frac{k-1}{m}\right\rceil} \epsilon, \label{nspi:eq6}
\end{align}}
which leads to the second bound by taking $k = \left \lceil \frac{\log{\frac{2\vmax}{\epsilon}}}{1-\gamma} \right \rceil$. Last but not least, starting back on \refequ{nspi:eq4}, and using the fact that $(I-\Gamma_{k-i,m})^{-1}=I+\Gamma_{k-i,m}(I-\Gamma_{k-i,m})^{-1}$ we see that:
\begin{align}
& \mu(v_{\pi_*}-v_{\pi_{k}}) - \gamma^k \vmax ~\le~ \sum_{i=0}^{k-1} \mu (\gamma P_{\pi_*})^i e_{k-i} ~+ \\
& + \sum_{i=0}^{k-1} \mu (\gamma P_{\pi_*})^i\Gamma_{k-i,m}(I-\Gamma_{k-i,m})^{-1} e_{k-i}.
\end{align}
The first term of the r.h.s. can be bounded exactly as for \PSDP. For the second term, we have:
\begin{align}
&\sum_{i=0}^{k-1} \mu (\gamma P_{\pi_*})^i\Gamma_{k-i,m}(I-\Gamma_{k-i,m})^{-1} e_{k-i} \\
& \le \sum_{i=0}^{k-1} \sum_{j=1}^\infty \gamma^{i+jm} c(i+jm) \epsilon_{k-i} \\
& = \gamma^m \sum_{i=0}^{k-1} \sum_{j=0}^\infty \gamma^{i+jm} c(i+(j+1)m) \epsilon_{k-i},
\end{align}
and we follow the same lines as above (from \refequ{nspi:eq3} to Equations~\eqref{nspi:eq5} and~\eqref{nspi:eq6}) to conclude.

\paragraph{\underline{CPI, CPI($\alpha$), API($\alpha$):}}
Conservative steps are addressed by a tedious generalization of the proof for API by \citet{munos2003}. Due to lack of space, the proof is deferred to the Supplementary Material.

\section{Proofs for Figure~\ref{fig:constants}}
\label{coefcompar}

We here provide details on the order relation for the concentrability coefficients.

\paragraph{\underline{$\Cpi \rightarrow \Cpia$}:}
 (i) We have $\Cpi \le \Cpia$ because
\begin{align}
d_{\pi_*,\mu} &=  (1-\gamma) \mu (I-\gamma P_{\pi_*})^{-1} 
 = (1-\gamma) \sum_{i=0}^\infty \gamma^i \mu(P_{\pi_*})^i \\
& \le (1-\gamma) \sum_{i=0}^\infty \gamma^i c_{\pi_*}(i) \nu = \Cpia \nu
\end{align}
and  $\Cpi$ is the smallest coefficient $C$ satisfying $d_{\pi_*,\mu} \le C \nu$.
 (ii) We may have $\Cpi<\infty$ and $\Cpia=\infty$ by designing a MDP on $\N$ where $\pi_*$ induces a deterministic transition from state $i$ to state $i+1$.

\paragraph{\underline{$\Cpia \rightarrow \Ca$:}}
(i) We have $\Cpia \le \Ca$ because for all $i$,  $c_{\pi_*}(i) \le c(i)$. 
(ii) It is easy to obtain  $\Cpia<\infty$ and $\Ca=\infty$ since $\Cpia$ only depends on \emph{one} policy while $\Cpia$ depends on \emph{all} policies.

\paragraph{\underline{$\Ca \rightarrow \Cbm{m}$ and $\Cam{m} \rightarrow \Cbbm{m}$:}}
(i) $\Cam{m} \le \frac{1}{1-\gamma^m} \Cbbm{m}$ holds because
{\small \begin{align}
\frac{\Cam{m}}{1-\gamma} = \sum_{i=0}^\infty \gamma^i c(i+m) 
& \le \sum_{i=0}^\infty \sum_{j=0}^\infty \gamma^{i+jm} c(i+(j+1)m)\\
& =\frac{1}{(1-\gamma)(1-\gamma^m)}\Cbbm{m}.
\end{align}}
(ii) One may have $\Cam{m}<\infty$ and $\Cbbm{m}=\infty$ when $c(i) = \Theta(\frac{1}{i^2 \gamma^i})$, since the generic term of $\Cam{m}$ is $\Theta(\frac 1 {i^2})$ (the sum converges) while that of $\Cbbm{m}$ is $\Theta(\frac 1 {i})$ (the sum diverges).
The reasoning is similar for the other relation.

\paragraph{\underline{$\Cam{m}\rightarrow \Ca$ and $\Cbbm{m}\rightarrow \Cbm{m}$:}}
We here assume that $m<\infty$.
(i) We have $\Cam{m} \le \frac{1}{\gamma^m}\Ca$ and $\Cbbm{m} \le \frac{1}{\gamma^m}\Cbm{m}$.
(ii) It suffices that $c(j)=\infty$ for some $j<m$ to have $\Cbm{m}=\infty$ while $\Cbbm{m}<\infty$, or to have $\Ca=\infty$ while $\Cam{m}<\infty$.

\paragraph{\underline{$\Cb \leftrightarrow \Cbm{m}$:}}
(i) We clearly have $\Cbm{m} \le \frac{1-\gamma^m}{1-\gamma} \Cb$. (ii) $\Cbm{m}$ can be rewritten as follows:
$$
\Cbm{m}=(1-\gamma)(1-\gamma^m) \sum_{i=0}^{\infty} \left(1+\left\lfloor \frac {i}m\right\rfloor\right)\gamma^{i} c(i). 
$$
Then, using the fact that $1+\left\lfloor \frac {i}m \right\rfloor \ge \max\left(1,\frac{i}{m}\right)$, we have
\begin{align}
& \frac{1-\gamma}{1-\gamma^m}\Cbm{m}  ~\ge~ \sum_{i=0}^{\infty} \max\left(1,\frac{i}{m}\right) \gamma^{i} c(i)  \\
& ~\ge~ \sum_{i=0}^{m-1} \gamma^{i} c(i) +  \sum_{i=m}^{\infty} \frac{i}{m} \gamma^{i} c(i) \\
&~\ge~ \sum_{i=0}^{m-1} \gamma^{i} c(i) +  \frac{m}{m+1} \sum_{i=m}^{\infty} \frac{i+1}{m} \gamma^{i} c(i) \\
&~=~ \sum_{i=0}^{m-1} \gamma^{i} c(i) + \frac{m}{m+1} \left(\Cb-\sum_{i=0}^{m-1} \gamma^{i} c(i)\right) \\
& ~=~ \frac{m}{m+1}\Cb+\frac{1}{m+1}\sum_{i=0}^{m-1} \gamma^{i} c(i).
\end{align}
Thus, when $m$ is finite, $\Cbm{m}<\infty \Rightarrow \Cb<\infty$.

\bibliographystyle{icml2014}
\bibliography{biblio}

\newpage

\onecolumn

\begin{center}
\Large{Supplementary Material}
\end{center}

\section{Proof for CPI, CPI($\alpha$), API($\alpha$)}

We begin by proving the following result:
\begin{theorem}
At each iteration $k < k^*$ of CPI (\refalgo{cpi}), the expected loss satisfies:
\begin{align}
\mu (v_{\pi_*}-v_{\pi_k}) & \le \frac{\Ca}{(1-\gamma)^2}\sum_{i=1}^k \alpha_i \epsilon_i + e^{\left\{(1-\gamma)\sum_{i=1}^k \alpha_i\right\}} \vmax.
\end{align}
\end{theorem}
\begin{proof}
Using the facts that $T_{\pi_{k+1}} v_{\pi_k} = (1-\alpha_{k+1})v_{\pi_k}+\alpha_{k+1} T_{\pi_{k+1}} v_{\pi_k}$ and the notation $e_{k+1}=\max_{\pi'}T_{\pi'} v_{\pi_k} - T_{\pi'_{k+1}} v_{\pi_k}$, we have:
\begin{align}
v_{\pi_*}- v_{\pi_{k+1}} &= v_{\pi_*}- T_{\pi_{k+1}} v_{\pi_k} + T_{\pi_{k+1}} v_{\pi_k} - T_{\pi_{k+1}} v_{\pi_{k+1}} \\
& = v_{\pi_*}- (1-\alpha_{k+1})v_{\pi_k} - \alpha_{k+1} T_{\pi'_{k+1}} v_{\pi_k} + \gamma P_{\pi_{k+1}} (v_{\pi_k}-v_{\pi_{k+1}}) \\
& = (1-\alpha_{k+1}) (v_{\pi_*}- v_{\pi_k}) + \alpha_{k+1} (T_{\pi_*}v_{\pi_*}-  T_{\pi_*}v_{\pi_k}) +  \alpha_{k+1} (T_{\pi_*}v_{\pi_k} -  T_{\pi'_{k+1}} v_{\pi_k}) + \gamma P_{\pi_{k+1}} (v_{\pi_k}-v_{\pi_{k+1}}) \\
& \le \left[ (1-\alpha_{k+1})I+\alpha_{k+1}\gamma P_{\pi_*} \right](v_\pi-v_{\pi_k}) + \alpha_{k+1} e_{k+1} + \gamma P_{\pi_{k+1}} (v_{\pi_k}-v_{\pi_{k+1}}). \label{cpi:eq00}
\end{align}
Using the fact that $v_{\pi_{k+1}}=(I-\gamma P_{\pi_{k+1}})^{-1}r$, and the fact that $(I-\gamma P_{\pi_{k+1}})^{-1}$ is non-negative, we can see that
\begin{align}
v_{\pi_k}-v_{\pi_{k+1}} & = (I-\gamma P_{\pi_{k+1}})^{-1} (v_{\pi_k}-\gamma P_{\pi_{k+1}}v_{\pi_k}-r) \\
& = (I-\gamma P_{\pi_{k+1}})^{-1} (T_{\pi_k} v_{\pi_k}-T_{\pi_{k+1}}v_{\pi_k}) \\
& \le (I-\gamma P_{\pi_{k+1}})^{-1} \alpha_{k+1} e_{k+1}.
\end{align}
Putting this back in \refequ{cpi:eq00}, we obtain:
\begin{align}
v_{\pi_*}- v_{\pi_{k+1}} \le \left[ (1-\alpha_{k+1})I+\alpha_{k+1}\gamma P_{\pi_*} \right](v_\pi-v_{\pi_k}) + \alpha_{k+1} (I-\gamma P_{\pi_{k+1}})^{-1} e_{k+1}.
\end{align}
Define the matrix $Q_k=\left[ (1-\alpha_{k})I+\alpha_{k}\gamma P_{\pi_*} \right]$, the set ${\cal N}_{i,k}=\{j; k-i+1 \le j \le k\}$ (this set contains exactly $i$ elements), the matrix $R_{i,k}=\prod_{j \in {\cal N}_{i,k}} Q_j$, and the coefficients $\beta_k=1-\alpha_k(1-\gamma)$ and $\delta_k=\prod_{i=1}^k \beta_k$. By repeatedly using the fact that the matrices $Q_k$ are non-negative, we get by induction
\begin{align}
v_{\pi_*}- v_{\pi_{k}} \le \sum_{i=0}^{k-1} R_{i,k}  \alpha_{k-i} (I-\gamma P_{\pi_{k-i}})^{-1} e_{k-i} + \delta_k \vmax. \label{cpi:eq0}
\end{align}
Let ${\cal P}_j({\cal N}_{i,k})$ be the set of subsets of  ${\cal N}_{i,k}$ of size $j$. With this notation we have
\begin{align}
R_{i,k} = \sum_{j=0}^{i} \sum_{I \in {\cal P}_j({\cal N}_{i,k})} \zeta_{I,i,k} (\gamma P_{\pi_*})^j
\end{align}
where for all subset $I$ of ${\cal N}_{i,k}$, we wrote 
\begin{align}
\zeta_{I,i,k}=\left(\prod_{n \in I} \alpha_n \right) \left(\prod_{n \in {\cal N}_{i,k} \backslash I}(1-\alpha_n)\right).
\end{align}
Therefore, by multiplying \refequ{cpi:eq0} by $\mu$, using the definition of the coefficients $c(i)$, and the facts that $\nu \le (1-\gamma)d_{\nu,\pi_{k+1}}$, we obtain:
\begin{align}
\mu(v_{\pi_*} - v_{\pi_{k}}) & \le \frac{1}{1-\gamma} \sum_{i=0}^{k-1} \sum_{j=0}^{i} \sum_{l=0}^{\infty} \sum_{I \in {\cal P}_j({\cal N}_{i,k})} \zeta_{I,i,k} \gamma^{j+l}c(j+l) \alpha_{k-i} \epsilon_{k-i} + \delta_k \vmax. \\
& = \frac{1}{1-\gamma}  \sum_{i=0}^{k-1} \sum_{j=0}^{i} \sum_{l=j}^{\infty} \sum_{I \in {\cal P}_j({\cal N}_{i,k})} \zeta_{I,i,k} \gamma^{l}c(l) \alpha_{k-i} \epsilon_{k-i} + \delta_k \vmax \\
& \le \frac{1}{1-\gamma} \sum_{i=0}^{k-1} \sum_{j=0}^{i} \sum_{l=0}^{\infty} \sum_{I \in {\cal P}_j({\cal N}_{i,k})} \zeta_{I,i,k} \gamma^{l}c(l) \alpha_{k-i} \epsilon_{k-i} + \delta_k \vmax \\
 &= \frac{1}{1-\gamma} \left( \sum_{l=0}^{\infty}\gamma^{l} c(l) \right) \sum_{i=0}^{k-1} \left( \sum_{j=0}^{i} \sum_{I \in {\cal P}_j({\cal N}_{i,k})} \zeta_{I,i,k} \right) \alpha_{k-i} \epsilon_{k-i} + \delta_k \vmax \\
& = \frac{1}{1-\gamma} \left( \sum_{l=0}^{\infty}\gamma^{l} c(l) \right) \sum_{i=0}^{k-1} \left( \prod_{j \in  {\cal N}_{i,k}} (1-\alpha_j+\alpha_j) \right) \alpha_{k-i} \epsilon_{k-i} + \delta_k \vmax \\
& = \frac{1}{1-\gamma}  \left( \sum_{l=0}^{\infty}\gamma^{l} c(l) \right) \left( \sum_{i=0}^{k-1} \alpha_{k-i} \epsilon_{k-i} \right)  + \delta_k \vmax.
\end{align}

Now, using the fact that for $x \in (0,1)$, $\log(1-x) \le -x$, we can observe that
\begin{align}
\log \delta_k &= \log \prod_{i=1}^k \beta_i 
 = \sum_{i=1}^k \log \beta_i 
 = \sum_{i=1}^k \log(1-\alpha_i(1-\gamma)) 
 \le -(1-\gamma)\sum_{i=1}^k \alpha_i.
\end{align}
As a consequence, we get $\delta_k \le e^{-(1-\gamma)\sum_{i=1}^k \alpha_i}$.
\end{proof}

In the analysis of CPI,  \citet{Kakade2002} show that the learning steps that ensure the nice performance guarantee of CPI satisfy $\alpha_k \ge \frac{(1-\gamma)\epsilon}{12 \gamma \vmax}$, the right term $e^{\left\{(1-\gamma)\sum_{i=1}^k \alpha_i\right\}}$ above tends $0$ exponentially fast, and we get the following corollary that shows that CPI has a performance bound with the coefficient $\Ca$ of API in a number of iterations $O\left(\frac{\log \frac 1 \epsilon}{\epsilon}\right)$.
\begin{corollary}
The smallest (random) iteration $k^{\dag}$ such that $\frac{\log \frac{\vmax}{\epsilon} }{1-\gamma} \le \sum_{i=1}^{k^{\dag}} \alpha_i \le \frac{\log \frac{\vmax}{\epsilon} }{1-\gamma}+1$ is such that $ k^{\dag} \le \frac{12 \gamma \vmax \log \frac{\vmax}{\epsilon} }{\epsilon(1-\gamma)^2}$ and the policy $\pi_{k^{\dag}}$ satisfies:
\begin{align}
\mu (v_{\pi_*}-v_{\pi_{k^{\dag}}}) & \le \left( \frac{\Ca\left(\sum_{i=1}^{k^{\dag}} \alpha_i \right) }{(1-\gamma)^2} +1 \right)  \epsilon \le \left(\frac{\Ca \left(\log \frac{\vmax}{\epsilon}  +1 \right)}{(1-\gamma)^3}+1 \right) \epsilon.
\end{align}
\end{corollary}

Since the proof is based on a generalization of the analysis of API and thus does not use any of the specific properties of CPI, it turns out that the results we have just given can straightforwardly be specialized
to CPI($\alpha$).
\begin{corollary}
Assume we run CPI($\alpha$) for some $\alpha \in (0,1)$, that is CPI (\refalgo{cpi}) with $\alpha_k=\alpha$ for all $k$. 
\begin{align}
\mbox{If }k = \left \lceil \frac{\log{\frac{ \vmax}{ \epsilon}}}{\alpha(1-\gamma)} \right \rceil,~~~\mbox{ then }& \mu (\v*-v_{\pi_k}) \le  \frac{\alpha(k+1)\Ca}{(1-\gamma)^2}\epsilon \le \left(\frac{\Ca \left(\log \frac{\vmax}{\epsilon}  +1 \right)}{(1-\gamma)^3}+1 \right) \epsilon.
\end{align}
\end{corollary}

The above bound for CPI($\alpha$) involves the factor $\frac{1}{(1-\gamma)^3}$. A precise examination of the proof shows that this amplification is due to the fact that the approximate greedy operator uses the distribution  $d_{\pi_{k},\nu} \ge (1-\gamma)\nu$ instead of $\nu$ (for API). In fact, using a very similar proof, it is easy to show that API($\alpha$) satisfies the following result.
\begin{corollary}
Assume API($\alpha$) is run for some $\alpha \in (0,1)$.
\begin{align}
\mbox{If }k = \left \lceil \frac{\log{\frac{ \vmax}{ \epsilon}}}{\alpha(1-\gamma)} \right \rceil,~~~\mbox{ then }& \mu (\v*-v_{\pi_k}) \le  \frac{\alpha(k+1)\Ca}{(1-\gamma)}\epsilon \le \left(\frac{\Ca \left(\log \frac{\vmax}{\epsilon}  +1 \right)}{(1-\gamma)^2}+1 \right) \epsilon.
\end{align}
\end{corollary}

\newcommand{\showcurvec}[8]{
\begin{figure}[ht!]
\begin{center}
\includegraphics[width=\taille\textwidth]{#1.pdf}
\end{center}
\begin{center}
\includegraphics[width=\taille\textwidth]{#3.pdf}
\end{center}
\begin{center}
\includegraphics[width=\taille\textwidth]{#5.pdf}
\end{center}
\caption{#7 Top: #2. Middle: #4. Bottom #6.\label{#8}}
\end{figure}
}

\section{More details on the Numerical Simulations}

\showcurvec{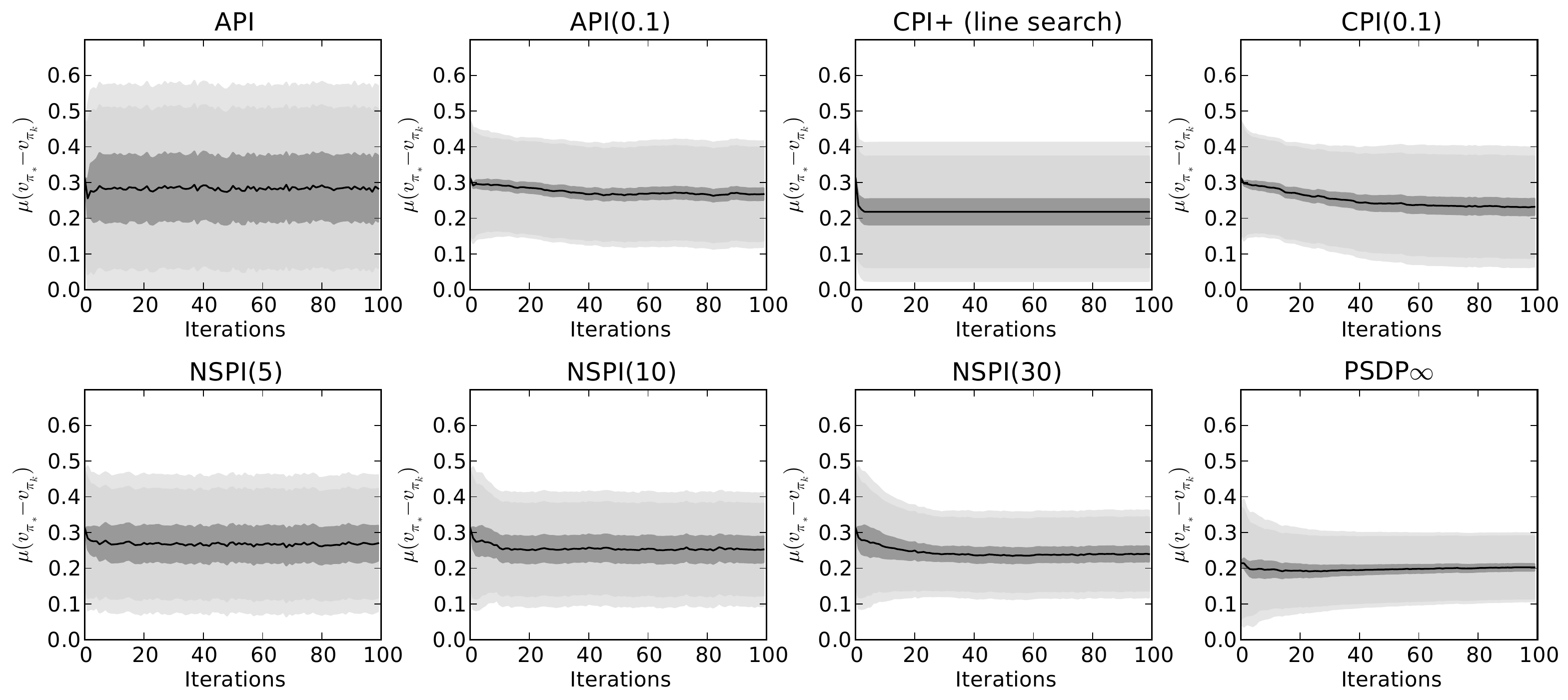}{$n_S=50$}{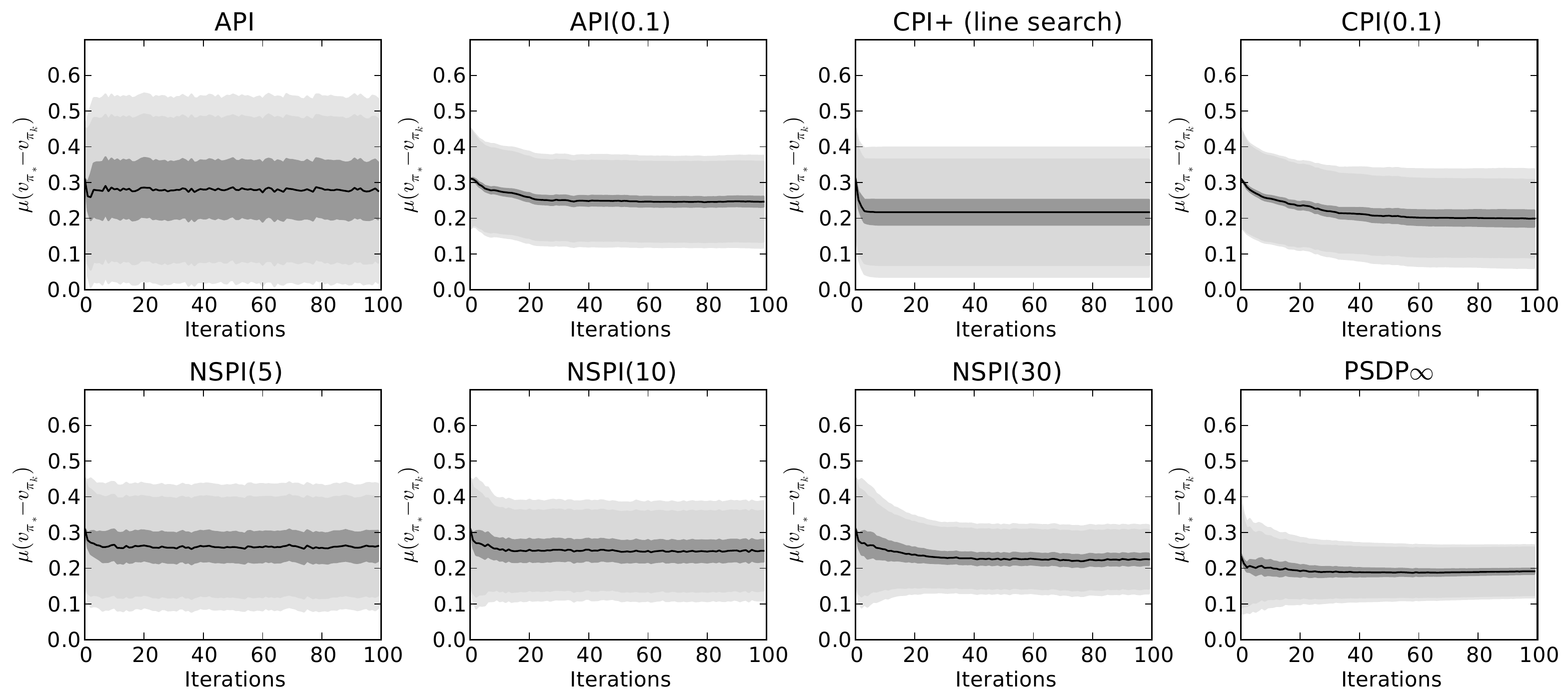}{$n_S=100$}{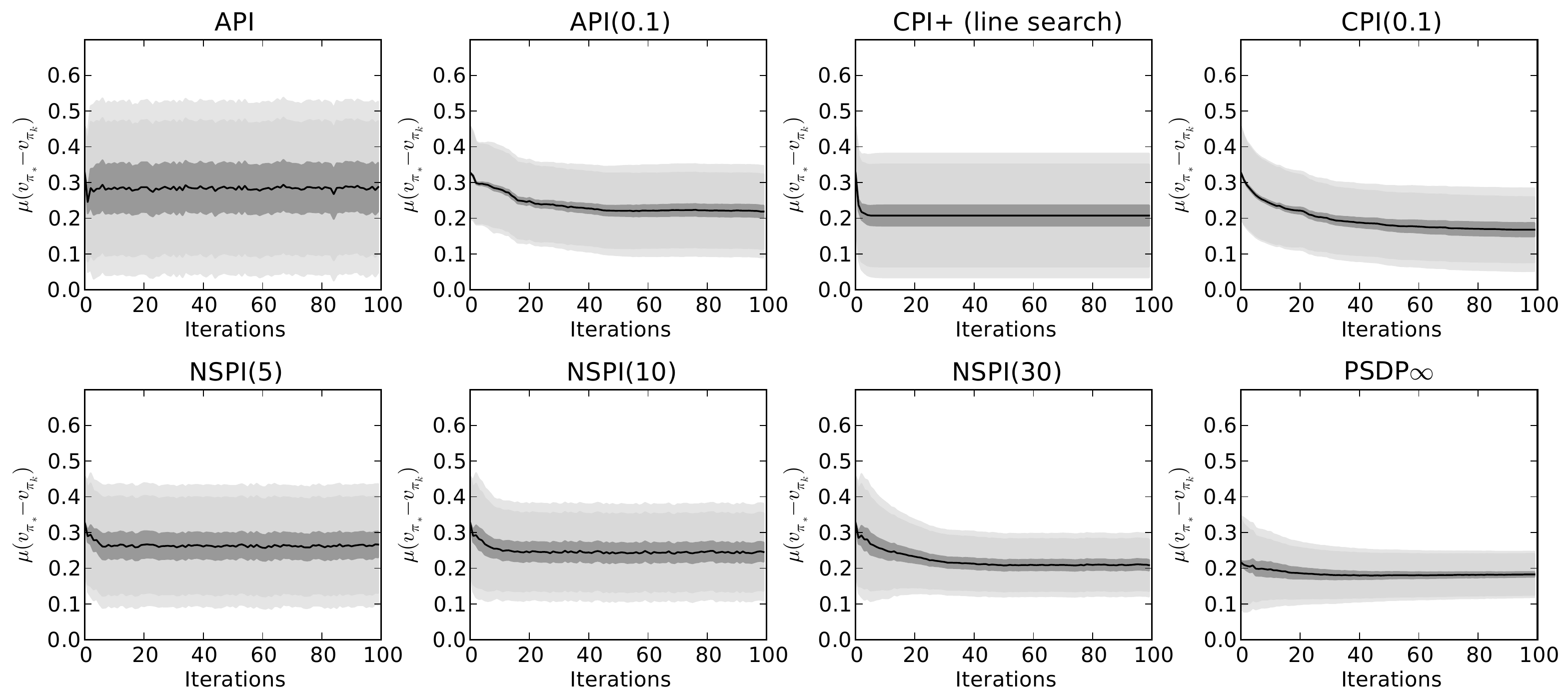}{$n_S=200$}{Statistics conditioned on the number of states.}{exps}

\showcurvec{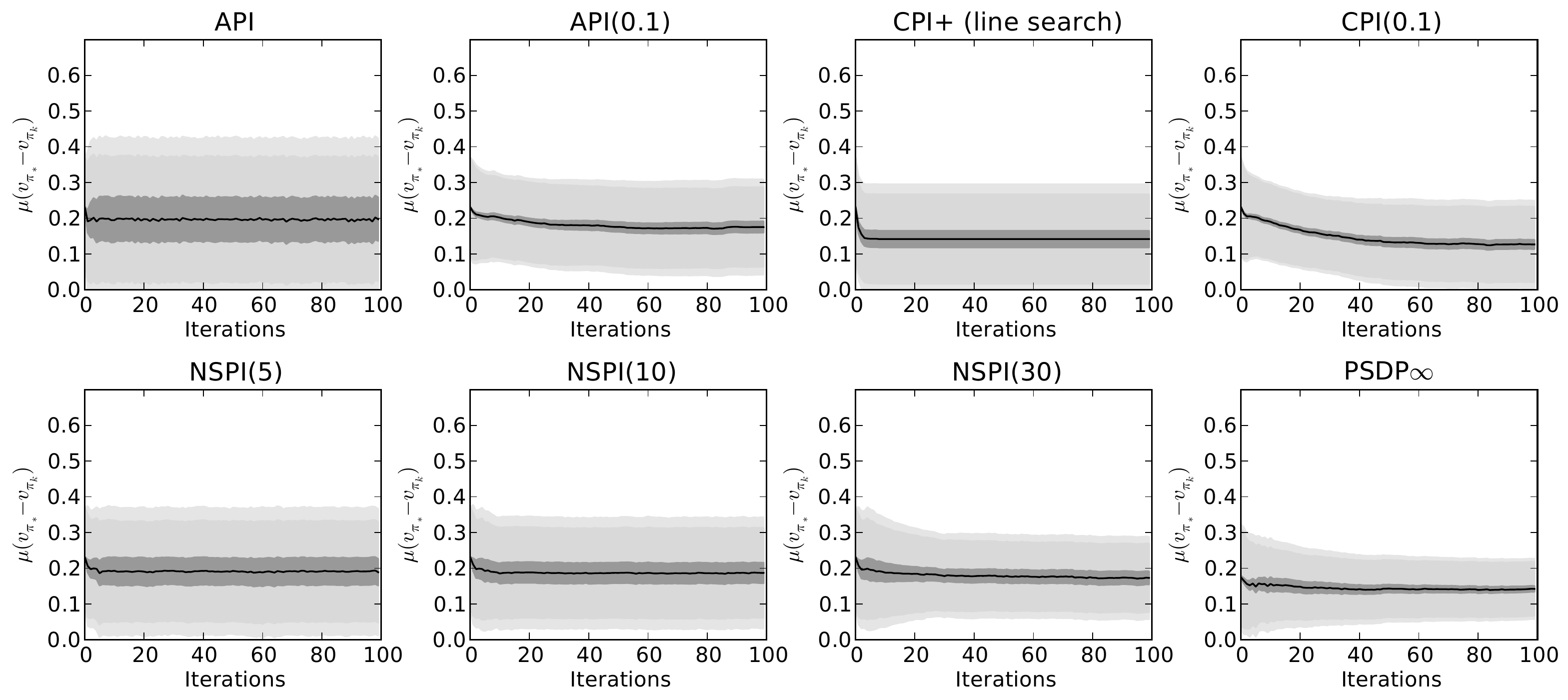}{$n_A=2$}{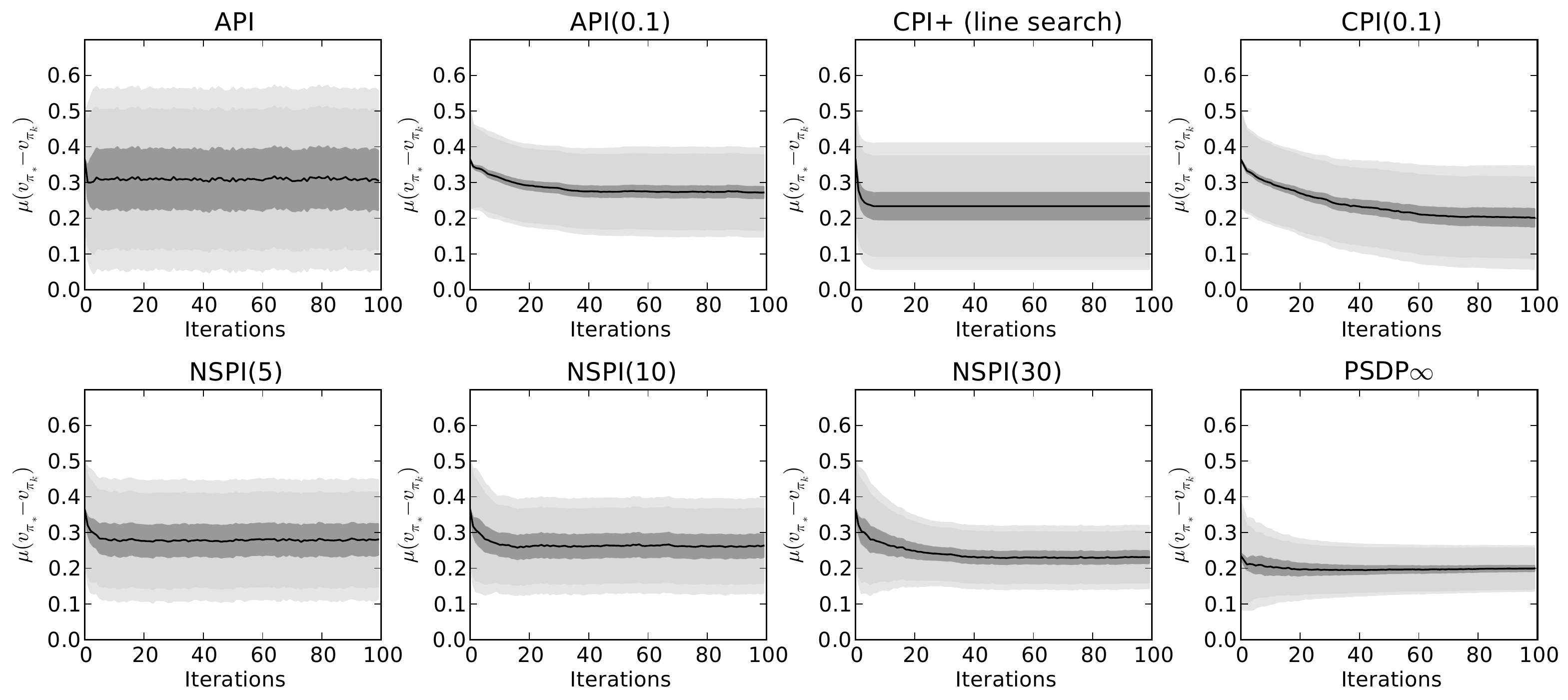}{$n_A=5$}{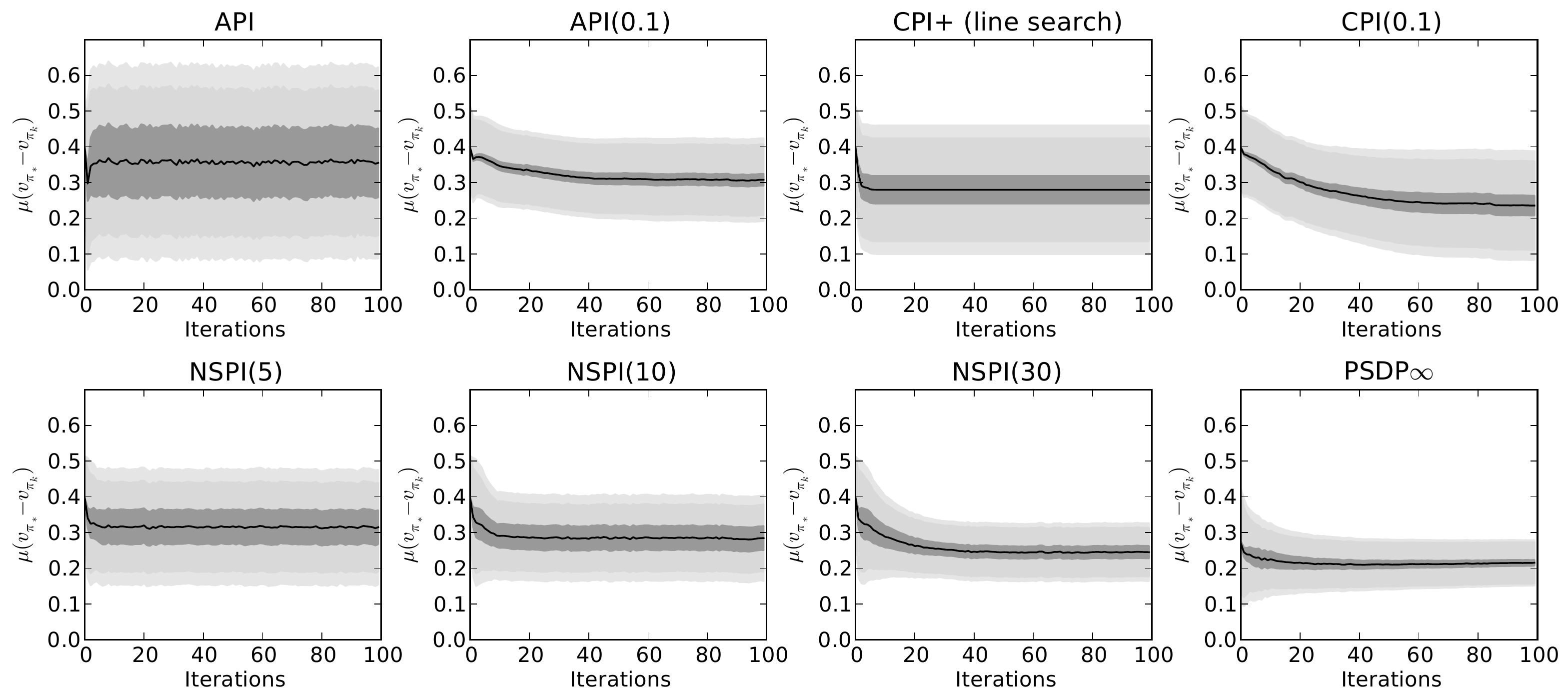}{$n_a=10$}{Statistics conditioned on the number of actions.}{expa}

\showcurvec{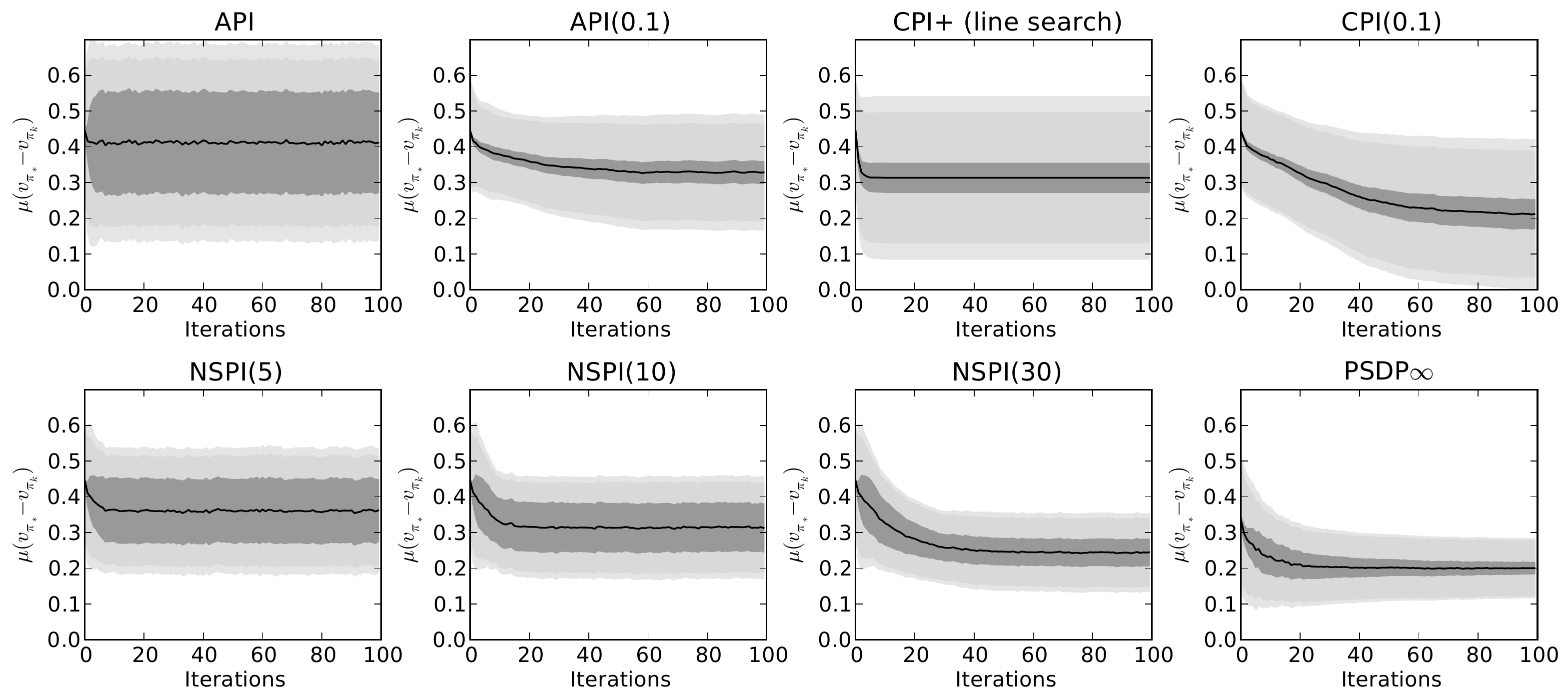}{$b=1$ (deterministic)}{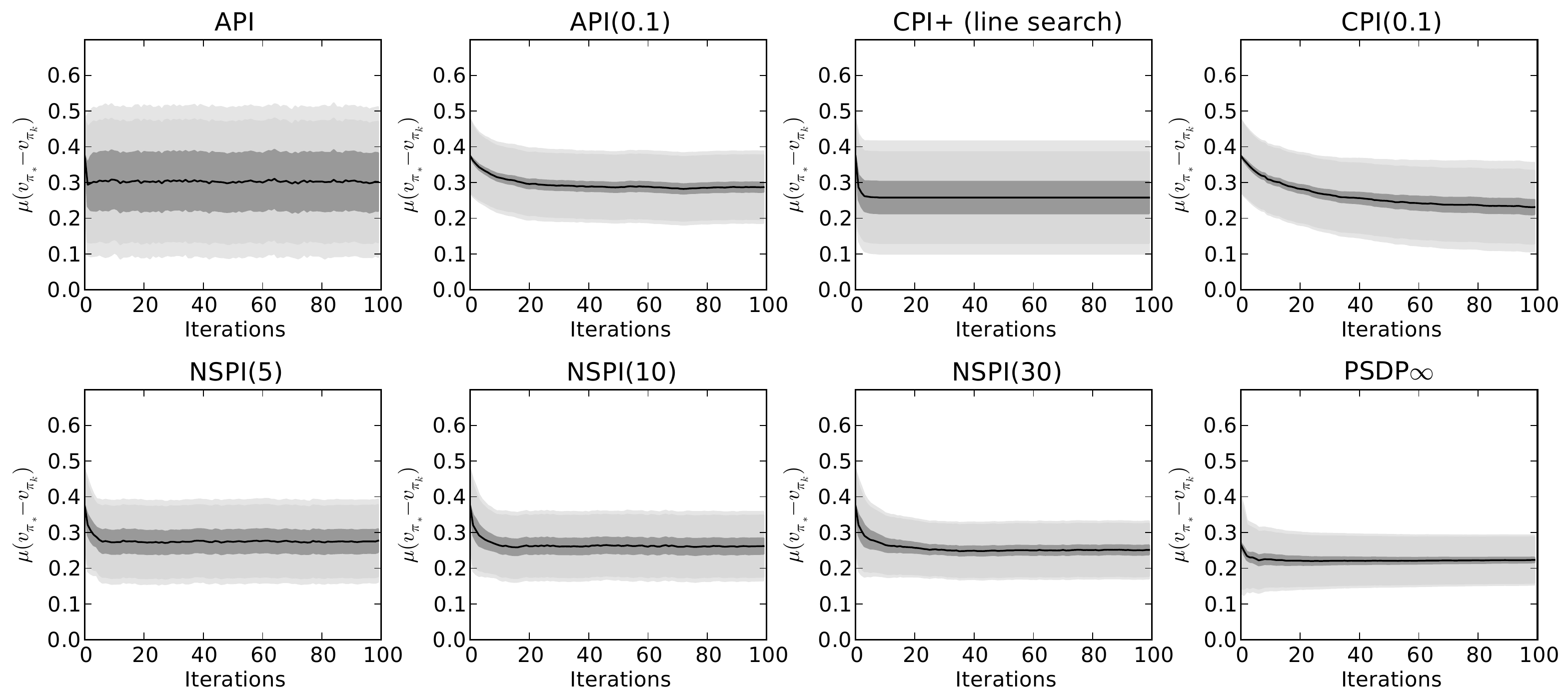}{$b=2$}{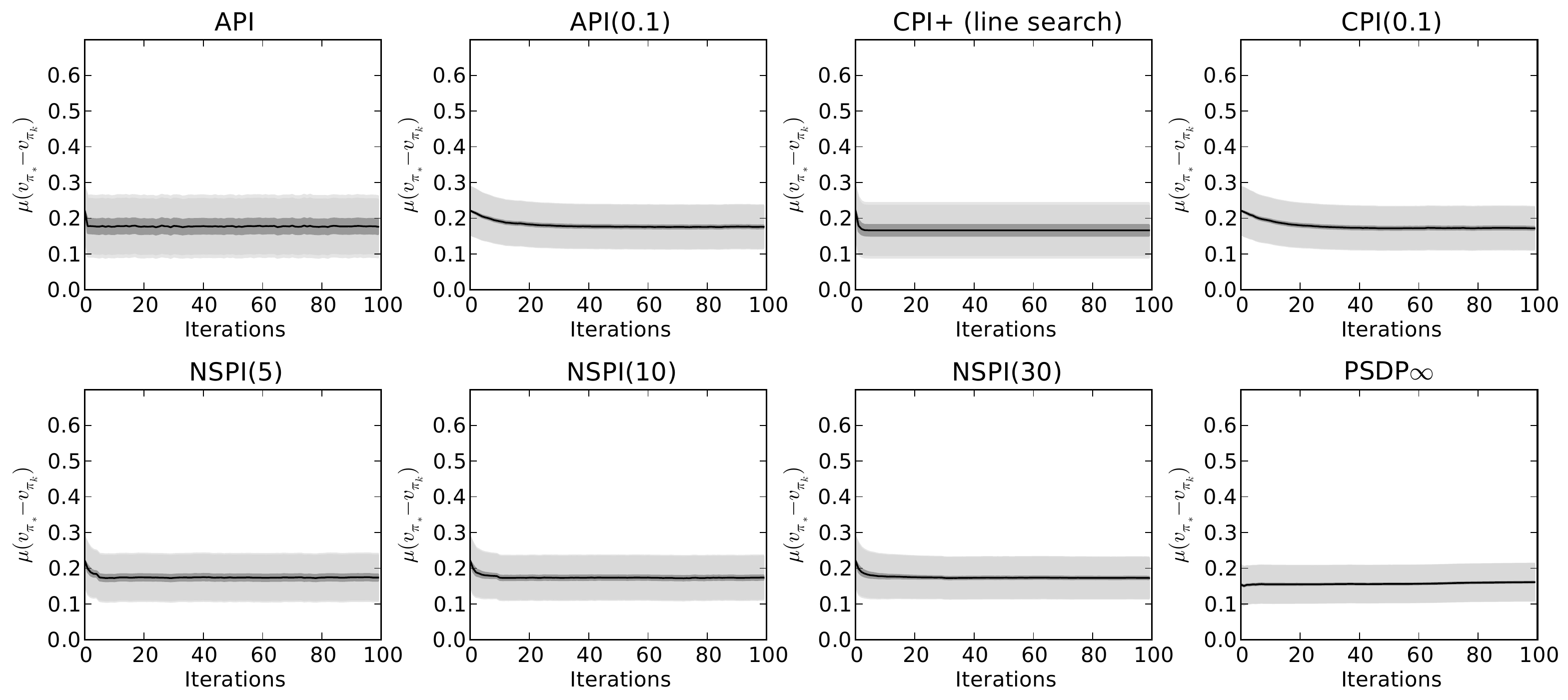}{$b=10$}{Statistics conditioned on the branching factor.}{expb}

\paragraph{Domain and Approximations}

In our experiments, a Garnet is
parameterized by 4 parameters and is written $G(n_S, n_A, b, p)$:
$n_S$ is the number of states, $n_A$ is the number of actions, $b$ is
a branching factor specifying how many possible next states are
possible for each state-action pair ($b$ states are chosen uniformly
at random and transition probabilities are set by sampling uniform
random $b-1$ cut points between 0 and 1) and $p$ is the number of
features (for linear function approximation). The reward is
state-dependent: for a given randomly generated Garnet problem, the
reward for each state is uniformly sampled between 0 and 1. Features
are chosen randomly: $\Phi$ is a $n_S\times p$ feature matrix of which
each component is randomly and uniformly sampled between 0 and 1. The
discount factor $\gamma$ is set to $0.99$ in all experiments.

All the algorithms we have discussed in the paper need to repeatedly
compute $\greedy_\epsilon(\rho,v)$ for some distribution $\rho=\nu$ or
$\rho=d_{\pi,\nu}$. In other words, they must be able to make calls to
an approximate greedy operator applied to the value $v$ of some policy
for some distribution $\rho$. To implement this operator, we compute a
noisy estimate of the value $v$ with a uniform white noise $u(\iota)$
of amplitude $\iota$, then projects this estimate onto the space
spanned by $\Phi$ with respect to the $\rho$-quadratic norm
(projection that we write $\Pi_{\Phi,\rho}$), and then applies the
(exact) greedy operator on this projected estimate. In a nutshell, one
call to the approximate greedy operator $\greedy_\epsilon(\rho,v)$
amounts to compute $\greedy
\Pi_{\Phi,\rho}(v+u(\iota))$.

\paragraph{Simulations}

We have run series of experiments, in which we callibrated the
perturbations (noise, approximations) so that the algorithm are
significantly perturbed but no too much (we do not want their behavior
to become too erratic). After trial and error, we ended up considering
the following setting. We used Garnet problems $G(n_S, n_A, b, p)$
with the number of states $n_S \in \{50,100,200\}$, the number of actions
$n_A \in \{2, 5, 10\}$, the branching factor $b \in \{ 1,2,10 \}
\}$ ($b=1$ corresponds to deterministic problems), the number of
features to approximate the value $p=\frac{n_S}{10}$, and the noise
level $\iota=0.1$ ($10\%$).

In addition to  Figure~\ref{expall} that shows the statistics overall for the
all the parameter instances, Figure~\ref{exps}, \ref{expa} and
\ref{expb} display statistics that are respectively conditioned on the
values of $n_S$, $n_A$ and $b$, which gives some insight
on the influence of these parameters.

\end{document}